\documentclass{article}

\PassOptionsToPackage{numbers, compress}{natbib}

\usepackage[preprint]{neurips_2023}

\usepackage[utf8]{inputenc} 
\usepackage[T1]{fontenc}    
\usepackage{hyperref}       
\usepackage{url}            
\usepackage{booktabs}       
\usepackage{amsfonts}       
\usepackage{nicefrac}       
\usepackage{microtype}      
\usepackage{xcolor}         

\usepackage{textcomp, gensymb}
\usepackage{graphicx}
\usepackage{caption}
\usepackage{subcaption}
\usepackage{wrapfig}
\usepackage{siunitx}
\usepackage{sidecap}
\usepackage[ruled,vlined]{algorithm2e}
\usepackage[page,header]{appendix}
\usepackage{titletoc}
\usepackage{adjustbox}
\usepackage{amsmath}
\usepackage{amsthm}

\DeclareMathOperator{\argmin}{argmin}

\newtheorem{theorem}{Theorem}

\newtheorem{lemma}{Lemma}

\title{Understanding and Improving Optimization in Predictive Coding Networks}

\author{%
  Nick Alonso\thanks{Author Contact: nalonso2@uci.edu} \\
  Cognitive Sciences Dept.\\
  University of California, Irvine\\
  \And
  Jeff Krichmar \\
  Cognitive Sciences Dept.\\
  Computer Science Dept.\\
  University of California, Irvine\\
  \And
  Emre Neftci \\
  Peter Gr\"unberg Institute, Forschungszentrum J\"ulich, Germany \\
  Electrical Engineering and Information Technology, RWTH Aachen, Germany
}

\begin{document}

\maketitle

\begin{abstract}
Backpropagation (BP), the standard learning algorithm for artificial neural networks, is often considered biologically implausible. In contrast, the standard learning algorithm for predictive coding (PC) models in neuroscience, known as the inference learning algorithm (IL), is a promising, bio-plausible alternative. However, several challenges and questions hinder IL's application to real-world problems. For example, IL is computationally demanding, and without memory-intensive optimizers like Adam, IL may converge to poor local minima. Moreover, although IL can reduce loss more quickly than BP, the reasons for these speedups or their robustness remains unclear. In this paper, we tackle these challenges by 1) altering the standard implementation of PC circuits to substantially reduce computation, 2) developing a novel optimizer that improves the convergence of IL without increasing memory usage, and 3) establishing theoretical results that help elucidate the conditions under which IL is sensitive to second and higher-order information.
\end{abstract}

\section{Introduction}
Artificial neural networks (ANNs) were originally created to mimic biological neural circuits \cite{mcculloch1943logical}. However, backpropagation (BP) \cite{rumelhart1995backpropagation}, the now standard algorithm used to train ANNs, is difficult to reconcile with neurobiology \cite{crick1989recent, lillicrap2020backpropagation}. Researchers in computational neuroscience and neuromorphic computing have attempted to alter BP to better fit what is known about the brain \cite{lillicrap2016random, liao2016important, guerguiev2017towards, richards2019dendritic}. However, these altered versions of BP typically either fail to overcome the biological implausibilities of BP, such as BP's segregated feed-back signals that do not alter feed-forward activities, or perform significantly worse than standard BP. One promising alternative approach is based on the predictive coding (PC) model in neuroscience \cite{rao1999predictive}. PC is a type of recurrent neural network typically trained with inference learning (IL). IL is a variant of generalized expectation maximization \cite{millidge2022theoretical}, which works by first performing inference, where an objective known as free energy is reduced w.r.t. neuron activities via the recurrent PC circuits. Then weights are updated to further reduce free energy. Unlike BP, IL does not require a segregated feedback stream. Instead, its recurrent circuitry requires feed-forward and feedback signals to interact with each other. Further, like the brain and in contrast to BP, IL performs Hebbian-like updates that are local in space and time \cite{whittington2017approximation}.

Previous work has found that networks trained with IL performed competitively with BP on classification and self-supervised tasks that use small and medium sized images \cite{whittington2017approximation, alonso2021tightening, salvatori2021associative, salvatori2022learning, alonso2022theoretical, song2022inferring}. Recent work has even found some noticeable performance advantages over BP/SGD (e.g., \cite{alonso2022theoretical, song2022inferring}), such as faster loss reduction and convergence with small mini-batches. These results suggest that IL may be useful in certain machine learning and neuromorphic computing applications, such as online learning scenarios which require small mini-batches. However, there are several challenges preventing IL from easily being applied to engineering problems: \textbf{1)} IL is much more computationally expensive than BP due primarily to IL's expensive inference phase. \textbf{2)} IL has only achieved comparable performance to BP when memory expensive optimizers like Adam are used. Without any optimizer, IL convergence is prone to fall in poor local minima \cite{alonso2022theoretical}, and the reason why this happens is not well understood. \textbf{3)} IL sometimes shows speedups in loss reduction and convergence over BP \cite{alonso2022theoretical, song2022inferring}, but the robustness of these speedups across models and tasks is unknown. In this paper, we tackle these challenges through the following contributions:
\begin{enumerate}
\item We develop and test a non-standard implementation of IL's inference phase, called sequential inference, which propagates errors more quickly through the network allowing for a significant reduction of the computational cost of IL.
\item We show that IL weight updates are drastically smaller in magnitude in early layers, which may lead the network to get caught in poor local minima near the initial parameters. We develop a custom optimizer called \textbf{M}atrix Update E\textbf{q}ualization (MQ) to address this problem. MQ requires no significant increase in memory or computation compared to IL and our simulations provide evidence that MQ prevents the convergence issues of IL at least as well as Adam.
\item We provide simulation results showing IL often reduces the loss more quickly than BP/SGD. We establish new theoretical results showing that IL is generally sensitive to second and higher-order information, which could explain this speed up.
\end{enumerate}
In sum, in our simulations below, IL with sequential inference and the MQ optimizer requires no more memory than BP/SGD, requires only slightly more computation than BP/SGD, converges to similar losses as BP/SGD, and is sensitive to higher-order information, which often leads to faster convergence than BP/SGD. \textit{As far as we know, this is the first time an energy-based learning algorithm (as defined in \cite{scellier2017equilibrium, whittington2019theories}) has performed as well or better than BP/SGD on natural images in the ways described above without using memory expensive optimizers and without requiring significantly more computation than BP.} This combined with the fact that IL is more bio-plausible than BP suggests IL, with the modifications made here, may be a promising learning algorithm for neuromorphic and bio-inspired machine learning communities.

\section{Background and Notation}

\subsection{Notation}
{\small
\begin{table}[h]
\centering
  \begin{tabular}{c l}
    \toprule
    Term & Description \\
     \toprule
    $W_l$ & Weight Matrix, pre-synaptic layer $l$\\
    $h_l$ & Feedforward Activity layer $l$, $h_l = W_{l-1} f(h_{l-1})$\\
    $\hat{h}_l$ & Optimized/Target Activity layer $l$\\
    $p_l$ & Local Prediction layer $l$, $p_l = W_{l-1}f( \hat{h}_{l-1})$\\
    $e_l$ & Local Error, layer $l$, $e_l = \hat{h}_l - p_l$\\
    \bottomrule
\end{tabular}
\caption{Notation}
\label{tab:notation}
\vspace{-10pt}
\end{table}}
Notation describing a multi-layered feed-forward (FF) network (MLP) is summarized in table \ref{tab:notation}. We assume a bias is stored in an extra column of each weight matrix $W_l$.

\subsection{Predictive Coding and Inference Learning}
Predictive coding (PC) networks are a kind of recurrent neural network typically trained with the inference learning algorithm (IL). IL can be interpreted as a variant of generalized Expectation Maximization algorithm (EM) \cite{friston2008hierarchical, millidge2022theoretical}. Like EM, IL proceeds in two steps: First, an energy function, known as free energy is minimized w.r.t. neuron activities. PC computations perform this minimization using gradient descent. Then weights are updated to further reduce free energy. 
Below, we present the energy function and equations for PC and IL weight updates.

\textbf{Free Energy:} The free energy, $F$, is defined here as follows:
\begin{equation}\label{eq:FEnergy}
    F = \mathcal{L}(y, \hat{h}_L) + \sum_{l=1}^{L} \gamma_l \frac12 \Vert \hat{h}_l - W_{l-1}f(\hat{h}_{l-1}) \Vert^2 + \sum_{l=1}^{L-1} \gamma^{decay}_l \frac{1}{2} \Vert f(\hat{h}_l) \Vert^2,
\end{equation}
where $\hat{h}$ are layer activities, $\mathcal{L}$ is the global loss, $\frac{1}{2} \Vert f(\hat{h}_l) \Vert^2$ is an optional regularization term, $f$ is a non-linearity, and $\gamma$ are positive scalar weighting terms. The middle term can be interpreted as a summation over squared prediction errors, where prediction $p_l = W_{l-1}f(\hat{h}_{l-1})$ and prediction errors are $e_l = \hat{h}_l - p_l$. If we set the activities at the output layer equal to the prediction target, $y = \hat{h}_L$ and ignore the decay (set $\gamma^{decay} = 0$), as is common in practice, the energy is just the sum of prediction errors: $F = \sum_{l=1}^{L} \gamma_l \frac12 \Vert e_l \Vert^2$.

\textbf{Inference:} During the inference phase, which is similar to the E-step of EM \cite{millidge2022theoretical}, neuron activities $\hat{h}$ initialized to feed-forward activities $h$ then are updated to reduce $F$. PC refers to the the process that updates activities iteratively using partial gradients of local prediction errors. For example, PC circuits update $\hat{h}_l$ using gradients from $e_l$ and $e_{l+1}$. Specifically, a single gradient update over a hidden layer $\hat{h}_l$ is
\begin{equation}\label{eq:actUpdate}
\Delta \hat{h}_l = \epsilon(-\frac{\partial F}{\partial \hat{h}_l}) = \epsilon( \gamma_{l+1} f'(\hat{h}_l) W_l^T e_{l+1} - \gamma_l e_l - \gamma^{decay}_l f'(\hat{h}_l) \hat{h}_l),
\end{equation}
where $\epsilon$ is the step size. The output layer activities, $\hat{h}_L$, are either fully clamped to output targets ($\hat{h}_L = y$) or are softly clamped such that each iteration $\hat{h}_L$ is updated to reduce the loss $\mathcal{L}(\hat{h}_L,y)$ and $\gamma_L e_L$. Here we assume $\mathcal{L}(\hat{h}_L,y) = \frac12\Vert y - p_L \Vert^2$, which yields a closed form solution that is computed at each iteration:
\begin{equation}\label{eq:outUpdate}
\hat{h}_L = \frac{1}{(1 + \gamma_L)} y + \frac{\gamma_L}{(1 + \gamma_L)} p_L.
\end{equation}
The derivation can be found in appendix \ref{app:derivOut}. In practice, gradient updates are typically performed for 15+ iterations (e.g., \cite{whittington2017approximation, alonso2021tightening, salvatori2022learning}).  

\textbf{Learning:} After the inference phase is completed, weights are updated to further reduce $F$. This update is typically performed by using the gradient of the local errors, e.g., $W_l$ is updated with the gradient of $\frac12\Vert e_{l+1}\Vert^2$:
\begin{equation}\label{eq:LMSUpdate}
\Delta W_l = -\alpha_l \frac{\partial F}{\partial W_l} = \alpha_l e_{l+1}f(\hat{h}_l)^T,
\end{equation}
where $e_{l+1} = \hat{h}_{l+1} - W_lf(\hat{h}_l)$ and $\alpha_l$ is a layer-wise step size. 

\subsection{Inference Learning Approximates Implicit Gradient Descent}
Previous works have analyzed the similarities between back-propagation (BP) and stochastic gradient descent (SGD) and IL (e.g., \cite{whittington2017approximation}). There, it was found that IL approaches BP/SGD as activities $\hat{h}$ approach feed forward activity values $h$ \cite{whittington2017approximation}. Although insightful, such analyses leave open the question of why IL is able to reduce the loss in a stable manner even when the $\hat{h}$ deviates significantly from $h$, which often occurs in practice \cite{rosenbaum2022relationship}. Recently, an alternative theoretical framework for IL has been established that avoids this issue. In particular, it was shown that, in the case of a mini-batch size of one and under specific conditions on the $\gamma$ terms in the energy equation and learning rates $\alpha_l$, IL is equivalent to \textit{implicit} stochastic gradient descent. These conditions on the $\gamma$ and $\alpha$ terms are non-standard but are approximated well in standard implementations (see \cite{alonso2022theoretical} and below). We emphasize that implicit SGD is not equivalent to standard SGD, which we call explicit SGD here. The difference between the two can stated as follows:
\begin{equation}
\begin{split}
\textbf{Explicit SGD: } \theta^{(b+1)} &= \theta^{(b)} - \beta \frac{\partial \mathcal{L}(\theta^{(b)})}{\partial \theta^{(b)}}\\
\textbf{Implicit SGD: } \theta^{(b+1)} &= \theta^{(b)} - \beta \frac{\partial \mathcal{L}(\theta^{(b+1)})}{\partial \theta^{(b+1)}},
\end{split}
\end{equation}
where $\theta$ are the model parameters, $b$ is the current training iteration, and $\beta$ is a `global' learning rate acting over all parameters. While standard/explicit SGD updates parameters with the gradient of the loss w.r.t. the current parameters, implicit SGD updates parameters with the gradient of the loss w.r.t. the parameters at the next training iteration, $b+1$. This gradient cannot be explicitly computed given known values at iteration $b$. Hence, this update is computed implicitly. It can be shown that the implicit SGD update is equivalent to the \emph{proximal update} \cite{parikh2014proximal, toulis2014implicit}:
\begin{equation}\label{eq:prox}
\theta^{(b+1)} = \argmin_{\theta} \mathcal{L}(\theta) + \frac{1}{2\beta} \Vert \theta - \theta^{(b)} \Vert^2 = \theta^{(b)} - \beta \frac{\partial \mathcal{L}(\theta^{(b+1)})}{\partial \theta^{(b+1)}}.
\end{equation}
The proximal algorithm sets the new parameters equal to the output of an optimization process that both minimizes the global loss $\mathcal{L}$ and the squared difference between the current and optimized parameters, which keeps the new parameters in the proximity of the old ones. The intuition for why IL approximates the proximal algorithm goes as follows: IL reduces $F$ w.r.t. to neuron activities before updating weights, which does two things: 1) $\hat{h}_L$ is updated to reduce the loss at the output layer (equation \ref{eq:outUpdate}), which means that when weights are updated the loss is reduced given the same input. 2) Local errors $e_l$ and activity $f(\hat{h}_l)$ magnitudes are minimized, which has the effect of minimizing the magnitude of the weight update since the weight updates are just outer products $e_{l+1}$ and $f(\hat{h}_l)$. Thus, IL's inference phase yields weight updates that both minimize loss and the update norm, just like the proximal update.

\section{Theoretical Results}

\subsection{IL with MQ and Implicit SGD}
\begin{wrapfigure}{R}{0.38\textwidth}
\centering
\includegraphics[width=0.35\textwidth]{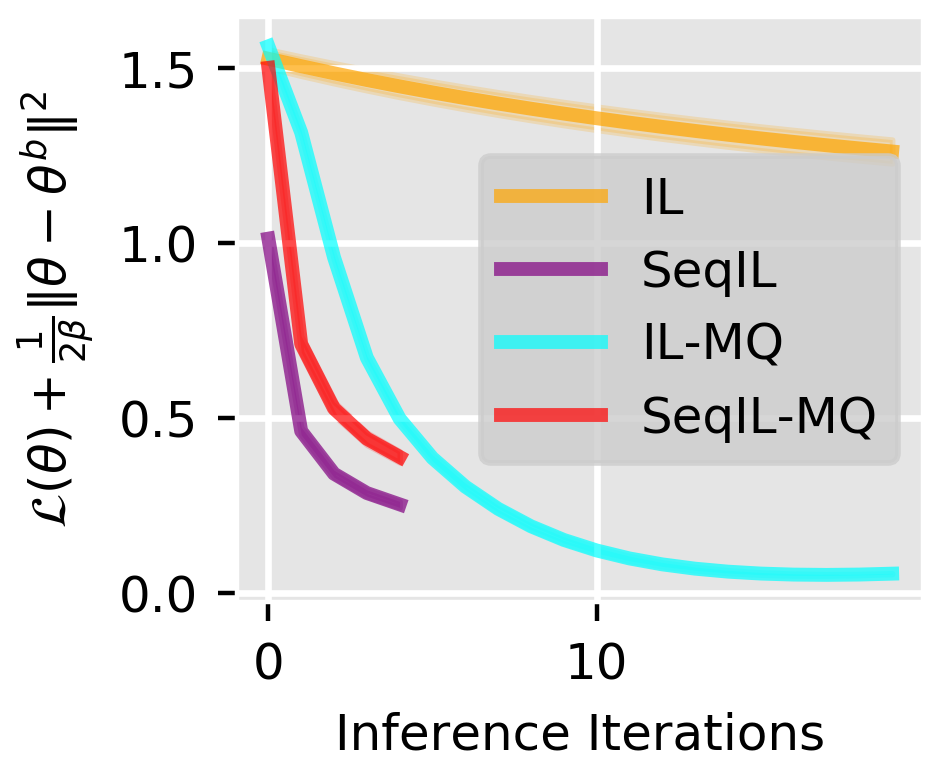}
\caption{Measurement of the proximal objective ($\beta=100$) during inference phase with fully clamped output layer of standard IL (IL) (20 iterations) and sequential IL (SeqIL) (5 iterations) with a fixed learning rate or MQ adaptive learning rates. Consistent with theorem \ref{thrm:impSGD=IL}, reducing energy F during the inference phase leads to a reduction of the proximal objective.}\label{fig:proxPart}
\vspace{-15pt}
\end{wrapfigure}
The IL algorithm was shown to be equivalent to the proximal algorithm/implicit SGD under a set of assumptions including that each weight update used a normalized step size: $\alpha_l = \Vert \hat{h}_l \Vert^{-2}$. This normalized step size solves the local prediction problem, such that local errors are minimized to zero given the same input \cite{alonso2022theoretical} with mini-batch size of 1. The original PC/IL algorithm \cite{rao1999predictive}, however, does not use this normalized step size and instead treats $\alpha_l$ as a static hyper-parameter. Further, below we develop an optimizer that uses static step sizes for each layer. This raises the question of how this class of IL algorithms relate to the proximal algorithm. We show that even when a non-normalized step size is used to update weights, IL approximates implicit SGD under certain settings of the $\gamma$ terms in the free energy (equation \ref{eq:FEnergy}) and in a certain limit concerning how much weight updates reduce local errors. In the next sections, we use the implicit SGD interpretation to further describe how IL is distinct from BP/SGD. 

Let a `static' scalar step size, $\alpha_l$, refer to a scalar that does not change during the inference phase with changes to $\hat{h}_l$ (unlike the normalized step size). Let $\theta_{IL}^{(b+1)}$ be the parameters updated  at iteration $b$ by the IL algorithm with static learning rates. Let $\theta_{prox}^{(b+1)}$ be the parameters updated by the proximal algorithm (equation \ref{eq:prox}). The next theorem states that the IL algorithm with static learning rates is equivalent to the proximal algorithm under specific settings of the $\gamma$ variables.
\begin{theorem}\label{thrm:impSGD=IL}
Consider the IL algorithm at training iteration $b$ with static $\alpha_l$ and mini-batch size one. Assume that at each inference iteration we update the $\gamma$ terms at hidden layers according to $\gamma_l = \alpha^2_{l-1} \Vert f(\hat{h}_{l-1}) \Vert^2$ and $\gamma^{decay}_l = \alpha_{l}^2 \Vert e_{l+1} \Vert^2$, and at the output layer $\gamma_L = \alpha_{L-1}(\frac{1}{\beta} + \Vert f(\hat{h}_{L-1}) \Vert^2) - 1$. In the limit where $\hat{h}^{(b)}_{L-1} \rightarrow h_{L-1}^{(b+1)}$, it is the case that $\theta_{IL}^{(b+1)} = \theta_{prox}^{(b+1)}$.
\end{theorem}
The theorem's proof is in appendix \ref{app:thrm1proof} and its extension to larger mini-batches is in appendix \ref{app:minib}. Thus, under these conditions, IL is an implementation of the proximal algorithm. In particular, IL implements the proximal algorithm in an indirect way by defining the parameters, $\theta$, optimized by the proximal update (eq. \ref{eq:prox}) as a function of hidden/auxiliary variables $\hat{h}$, i.e., $\theta = j(\hat{h})$, which is then optimized through updates to $\hat{h}$ (see \ref{app:thrm1proof} for details). The limit in which this theorem is true is the limit where the last hidden layer activity $\hat{h}_{L-1}^{(b)}$ approaches the feed-forward activity after weights are updated, $h_{L-1}^{(b+1)}$, given the same input $x^{(b)}$. This limit holds in the case where local errors are fully minimized by the weight updates. The proof also depends on particular settings for the $\gamma$ terms. In practice, these settings are usually not used, but we can see that the $\gamma_l$ terms will generally be positive scalars and are therefore approximated in practice by weighting each term with some positive scalar (see \cite{alonso2022theoretical} and figure \ref{fig:proxPart} and supplementary figure \ref{fig:proxFull}). Further, we can see that $\gamma^{decay} \approx 0$ at most hidden layers, since the magnitude of errors are initially zero and remain small afterward, providing justification for the common practice of neglecting the decay term in $F$. Finally, the $\beta$ term, which is the 'global' learning rate in the implicit SGD/proximal update, affects the clamping of the output layer: if $\beta \rightarrow 0$ then $\gamma_L \rightarrow \infty$ and $\hat{h}_L \rightarrow h_L$ (i.e., as $\beta$ gets smaller the output layer is more softly clamped). If $\beta \rightarrow \infty$ then $\gamma_L$ will get smaller and the output layer will become more strongly clamped. Indeed, in figure \ref{fig:proxPart}, we show empirically that several implementations of IL with a fully clamped output layer significantly reduce the proximal objective with a large $\beta$ during the inference phase (see fig. \ref{fig:proxFull} and appendix \ref{app:proxMethod} for more detail and evidence).

\subsection{IL is Sensitive to Second-Order Information}
Recent work found that IL, even without optimizers, often reduces loss more quickly than vanilla BP/SGD, especially when small mini-batches are used \cite{alonso2022theoretical}. Relatedly, it was found empirically by multiple works that IL often takes a shorter/more direct path toward local minima than BP/SGD \cite{alonso2022theoretical, song2022inferring}. However, it has not been clearly explained mathematically why and how these shorter paths and quicker convergence are achieved. Here we provide some explanation based on the insights of Toulis et al. \cite{toulis2014implicit, toulis2016stochastic}, who showed that implicit SGD is sensitive to second-order information for small learning rates. Let $\Delta \theta^{(b)}_{IL}$ equal the IL update at iteration $b$.

\begin{theorem}\label{thrm:2ndOrder}
Consider neural network parameters $\theta^{(b)}$ at iteration $b$ with mini-batch size 1. Assume weight-wise step sizes $\alpha_l$ and the $\gamma$ variables are set so that $\Delta \theta^{(b)}_{IL} = -\beta \frac{\partial \mathcal{L}(\theta^{(b+1)})}{\partial \theta^{(b+1)}}$. With these assumptions we have $\Delta \theta^{(b+1)}_{IL} = -\beta \frac{\partial \mathcal{L}(\theta^{(b+1)})}{\partial \theta^{(b+1)}} \approx -(I + \beta H)^{-1} \beta \frac{\partial \mathcal{L}(\theta^{(b)})}{\partial \theta^{(b)}}$ with error $\mathcal{O}(\beta^2)$, where $H = \frac{\partial^2 \mathcal{L}(\theta^{(b)})}{\partial \theta^{(b)2}}$.
\end{theorem}

For proof see appendix \ref{app:proofThrm2}. This theorem implies that implicit SGD, and IL when it approximates implicit SGD closely, is sensitive to second-order information. Specifically, IL/implicit SGD approximates a shrinked version of the gradient, where the shrinkage is determined by the inverse Hessian. This use of second-order information, about the curvature of the loss landscape, can help account for why IL often converges more quickly and takes more direct paths toward minima than standard SGD. This theorem also implies that when $\beta$ is small ($0 < \beta < 1$), IL closely approximates the Newton-Raphson algorithm with 'damping' (scaling by a small learning rate) and Levenburg-Marquardt regularization \cite{nesterov2006cubic}, which is: $\theta^{(b+1)} = \theta^{(b)} - \alpha (\textbf{I}\lambda + H^{-1}) \frac{\partial \mathcal{L}(\theta^{(b)})}{\partial \theta^{(b)}}$, where $\lambda$ is a scalar weighting term and $\alpha$ the learning rate/damping term. Thus, for small but non-zero $\beta$ and when IL approximates implicit SGD closely, IL approximates a regularized and damped version of the Newton-Raphson method with small error, $\mathcal{O}(\beta^2)$. Note that as $\beta \rightarrow \infty$ the approximation to Newton-Raphson becomes worse because IL becomes increasingly sensitive to even higher-order information in the Taylor-expansion. \textit{As far as we know, this is the first time this relation between IL and the Newton Raphson-method and higher order information has been shown}. We find empirically that our implementations of IL reduces the proximal objective well for both large and some small values of $\beta$ (fig. \ref{fig:proxPart} and \ref{fig:proxFull}), suggesting IL is generally sensitive to second and higher order information.

\subsection{Reinterpreting IL's Relation to BP and SGD}
Previous work found that in the limit where 
$\hat{h}_l \rightarrow h_l$, IL approaches BP/SGD \cite{whittington2017approximation, whittington2019theories, millidge2022backpropagation, rosenbaum2022relationship}. This limit is equivalent, in our notation, to the limit where the $\beta \rightarrow 0$ since in this case the output layer $\hat{h}_L$ approaches its initial value $h_L$, resulting in $\hat{h}_l \rightarrow h_l$. Notice that this is exactly what theorem \ref{thrm:2ndOrder} implies: as $\beta \rightarrow 0$ the Hessian term $\beta H$ and the error, $\mathcal{O}(\beta^2)$, go to zero, and we are left with a standard SGD update. \textit{Thus, consistent with these previous results, the implicit SGD interpretation predicts that IL should approach SGD as $\beta \rightarrow 0$ and $\hat{h} \rightarrow h$. However, theorems \ref{thrm:impSGD=IL} and \ref{thrm:2ndOrder} imply that standard/explicit SGD is not the best interpretation of IL when $\beta > 0$, since the SGD approximation has a larger error, $\mathcal{O}(\beta)$, than approximations that take into account higher order terms, which have at most error of $\mathcal{O}(\beta^2)$}.

\section{The Matrix Update Equalization Optimizer}
We present a novel optimizer, called Matrix Update Equalization (MQ), made specifically for the IL algorithm. Previous work (e.g. \cite{alonso2022theoretical}, also see below) observed that IL often converges to poor local minima when it does not use optimizers with parameter-wise adaptive learning rates, like Adam. However, Adam and similar optimizers are memory intensive. MQ seeks to provide similar benefits as Adam without any significant cost in memory or computation. It is motivated by two observations. First, Theorem 1 and 2 imply IL algorithms that only use matrix-wise scalar learning rates are sensitive to second and higher-order information. Second, IL updates, without the use of Adam, tend to be very small\footnote{It is worth noting that IL's tendency for small update magnitudes is consistent with the proximal interpretation.}, especially at early layers (see figure \ref{fig:wtAnlyze}), which likely explains why IL gets caught in shallow local minima nearer its initial values. Adam makes updates larger and more equal across matrices. These observations suggest that the main benefit of using optimizers like Adam with IL is that it prevents weight updates at early layers from becoming too small, pushing parameters further to deeper and flatter minima. This is contrary to the benefits Adam is supposed to provide BP/SGD, which is in large part to provide BP with an estimate of second order information \cite{kingma2014adam} to speed up training.
\begin{wrapfigure}{R}{0.36\textwidth}
\centering
\includegraphics[width=0.35\textwidth]{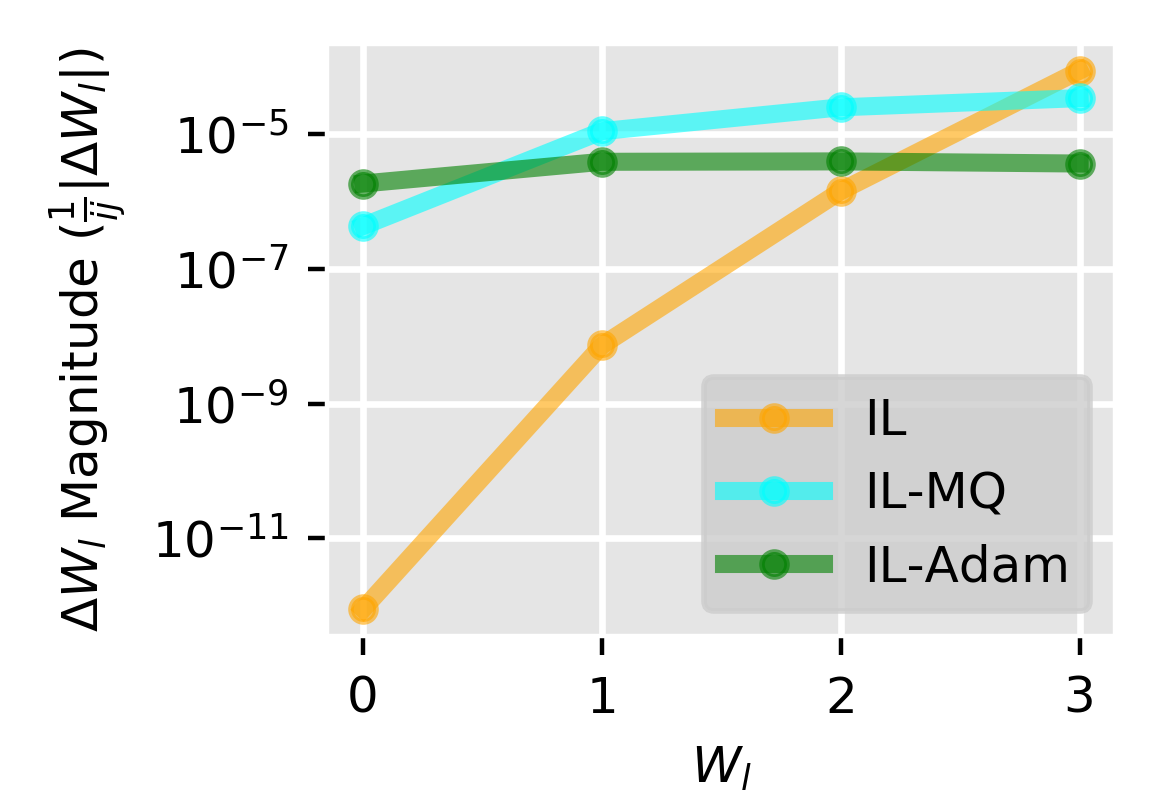}
\caption{Average update magnitude of each weight matrix in MLP with three hidden layers.}\label{fig:wtAnlyze}
\vspace{-10pt}
\end{wrapfigure}

Motivated by these observations, MQ uses only a single scalar learning rate per layer that adapts to prevent overly small update magnitudes and to better equalize the update magnitudes between matrices. MQ stores several scalars: a constant matrix-wise learning rate, $\alpha_l$, for each matrix $W_l$, a constant global minimum learning rate, $\alpha_{min}$, and matrix wise scalar $v_l$, which is updated continuously to adapt to the moving average of the update magnitude for matrix $W_l$. The MQ weight update for weight matrix $W_l$ is
\begin{equation}\label{eq:mqUpdate}
W_l^{(b+1)} = W_l^{(b)} - (\frac{\alpha_l}{v_l+r} + \alpha_{min}) \frac{\partial F}{\partial W_{l}^{(b)}},
\end{equation}
where $r$ is a constant that prevents division by zero and stabilizes learning. The minimum learning rate $\alpha_{min}$ is used to prevent learning from decreasing too much in any one matrix. The variable $v_l$ is updated each training iteration $b$ as follows:
\begin{equation}
v_l^{(b+1)} = (1 - \rho) v_l^{(b)} - \rho \frac{1}{ij} \vert \frac{\partial F}{\partial W_{l}^{(b)}} \vert,
\end{equation}
where $i$ and $j$ refer to the row and column sizes, respectively, of the gradient $\frac{\partial F}{\partial W_l}$ and $\vert \frac{\partial F}{\partial W_{l}^{(b)}} \vert$ is the L-1 norm of the gradient. Finally, $\rho$ is a global hyper-parameter (scalar). Note, although $\rho$ is a hyper-parameter, we use a method for reducing bias early in training, which is explained in appendix \ref{app:MQMethod}. Thus, $v_l$ first averages the update over elements of the absolute value of the gradient (right term), then it takes a weighted average of this with the previous value of $v_l$. Figure \ref{fig:wtAnlyze} shows MQ reduces the differences in update magnitude between weights matrices nearly as well as Adam and prevents updates at early layers form becoming too small.

\section{Reducing Computation with Sequential Inference}
\begin{figure}[t]
\includegraphics[width=.99\textwidth]{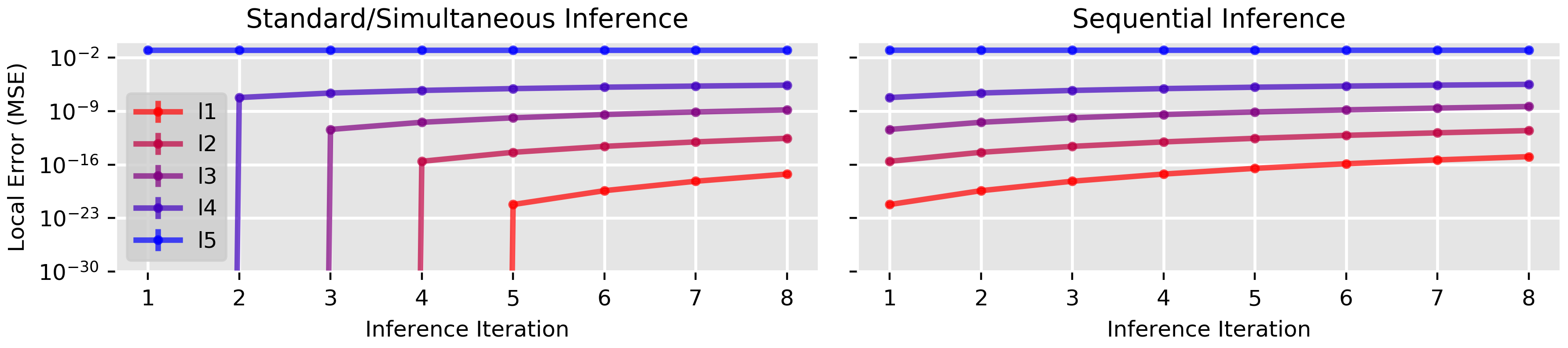}
\caption{Measurement of the mean of the squared local errors, $e_l$, at each hidden layer (layer 1-4) and the output layer (layer 5) in a fully connected network with four hidden layers.}
\centering
\label{fig:seqInfTest}
\end{figure}

The inference phase accounts for most of the computational cost of the IL algorithm. Typically, activities are updated for $T=15$ or more iterations (e.g., \cite{whittington2017approximation, alonso2021tightening, salvatori2022learning}), which requires around $T(2L-1)$ matrix multiplications. The obvious way to reduce computation is to truncate the inference phase (reduce $T$). The difficulty with this approach is that the standard implementation of the inference phase requires at least $T \geq (L-2)$ in order for non-zero errors to form at each hidden layer (see figure \ref{fig:seqInfTest}), and typically more iterations to attain the best performance (see experiments section), e.g., a network with four hidden layers typically requires at least $T \geq 4$, and usually much more than $4$ to perform well. In figure \ref{fig:seqInfTest} and \ref{fig:simVseqDiagram}, we show that the standard implementation requires large $T$ because there is a delay in error propagating through the network. Each inference iteration, $t$, the standard method computes and stores errors at each layer, then these stored errors are used to update the activities according to equation \ref{eq:actUpdate} (see algorithm \ref{alg:1}). This process simulates activities being updated simultaneously in time, since the update at one layer does not alter the updates or errors at other layers until the next iteration/time step. Instead of updating activities simultaneously, we propose updating them one at a time starting from the output layer and moving back in a sequence (see algorithm \ref{alg:2}). We call this method \textit{sequential inference}. Sequential inference removes the delay and progagates non-zero errors to every layer in a single iteration, allowing for smaller $T$ (see figure \ref{fig:seqInfTest} and supplementary figure \ref{fig:simVseqDiagram}). \footnote{Though this is a simple alteration, we could only find one work where an implementation similar to sequential inference is used \cite{rosenbaum2022relationship}, and in that work this implementation was not discussed in the paper, analyzed, or used as a method for reducing computation. Instead, it seemed to be an arbitrary implementation decision.}

\section{Experiments}
\begin{wrapfigure}{R}{0.4\textwidth}
\centering
\includegraphics[width=0.35\textwidth]{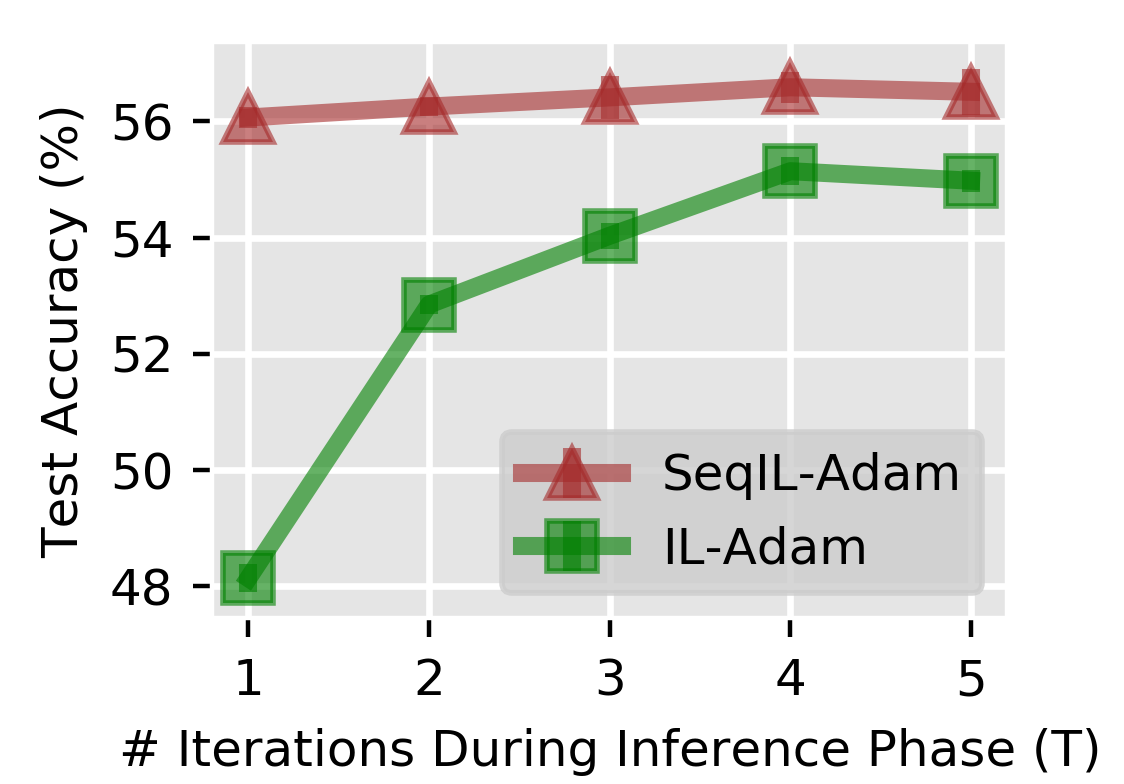}
\caption{MLPs trained on CIFAR-10 for 45 epochs, using either sequential inference with Adam (SeqIL-Adam) or simultaneous inference with Adam (SimIL-Adam). X-axis specifies the number of inference iterations used during training.}\label{fig:seqVsimAcc}
\vspace{-5pt}
\end{wrapfigure}
In this section, we test the performance of sequential inference and the MQ optimizer. The MQ optimizer works well in all models using following hyper-parameters: $\alpha_{min} = .001$, $r = .000001$, and $\rho = .9999$ for fully connected and $\rho =.999$ for convolutional networks. Grid searches are used to find learning rates, $\alpha_l$. Mini-batches size 64 are used in all simulations.

\textbf{Comparing Sequential and Simultaneous IL} The standard/simultaneous inference method creates a delayed error propagation through the network, as illustrated in figure \ref{fig:seqInfTest}. We test how this delay effects performance at very small values of $T$ on a classification of CIFAR-10 images. An MLP with four hidden layers, dimension 3072-4x1024-10, was trained on CIFAR-10 using both standard/simultaneous inference with Adam optimization (IL-Adam) and sequential inference with Adam optimization (SeqIL-Adam) for different values of $T$. Grid searches were used to find the learning rate and step size, $\epsilon$, for activity updates. Models are trained over 45 epochs, about the amount of time needed for learning to near convergence. Results shown in figure \ref{fig:seqVsimAcc}. SeqIL-Adam performs approximately the same for all values of $T$ (1-5), while IL-Adam's performance is highly sensitive to the value of $T$, with performance degrading significantly for smaller values of $T$.

\begin{figure}[t]
\includegraphics[width=\textwidth]{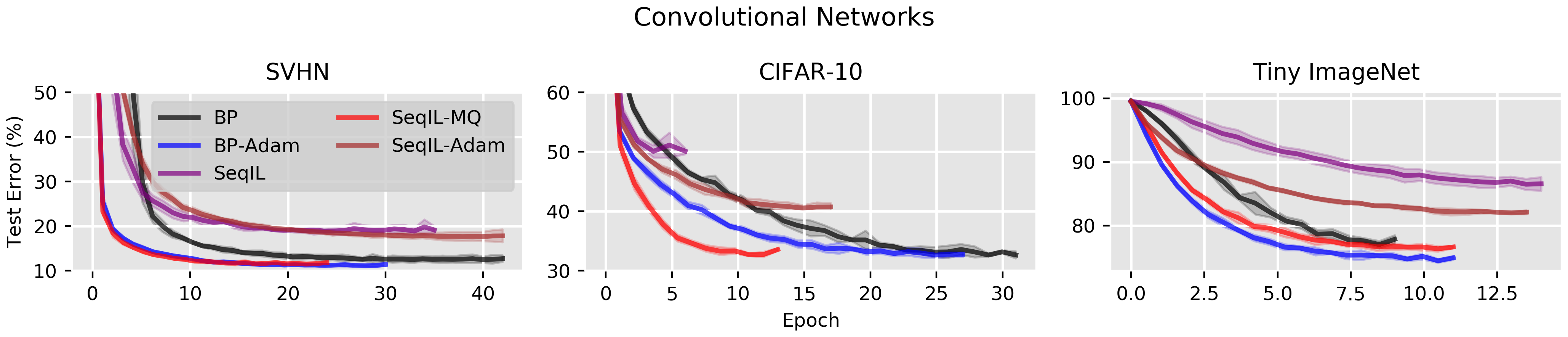}
\caption{Small convolutional networks trained on classification tasks. Training runs are averaged over 5 seeds. Shaded error bars show standard deviation across seeds. Each training run is shown up until convergence (the point of lowest test error +1 epoch).}
\centering
\label{fig:conv}
\end{figure}

\textbf{Classification with Natural Images} Next, we test the MQ optimizer combined with SeqIL on classifications tasks with natural image data sets: SVHN, CIFAR-10, and Tiny Imagenet. We train fully connected MLPs dimension 3072-3x1024-10 on SVHN and CIFAR-10, and small convolutional networks on SVHN, CIFAR-10, and Tiny Imagenet. All of IL algorithms use a highly truncated inference phases of T=3. In addition to sequential IL with MQ (SeqIL-MQ), we test SeqIL with no optimizer (SeqIL) and SeqIL with Adam (SeqIL-Adam). We compare to the performance of BP-SGD and BP with Adam (BP-Adam). Grid searches over learning rates were performed to find the learning rate that yielded the best test performance at convergence. SeqIL-MQ test accuracies were never significantly worse than BP and slightly better than BP on Tiny ImageNet, while consistently performing comparably to BP-Adam (table \ref{tab:Acc}). Further, SeqIL-MQ reduced loss more quickly early in training than BP-SGD in all simulations, and typically converged much more quickly (figure \ref{fig:conv}). SeqIL-MQ learning curves looked more similar to BP-Adam, despite the fact SeqIL-MQ uses about one third the memory of Adam per parameter. SeqIL-Adam matched the performance of SeqIL-MQ in fully connected networks but struggled in convolutional networks. We believe this is because in convolutional networks SeqIL-Adam had very small weight updates, which could not be increased by simply increasing learning rate due to instability (see app. fig \ref{fig:wtAnlyzConv}). Consistent with previous work on IL, we find that SeqIL without optimizers often reduces loss more quickly than BP early in training in MLPs, but struggled to converge to good minima.
\begin{figure}[b]
\includegraphics[width=\textwidth]{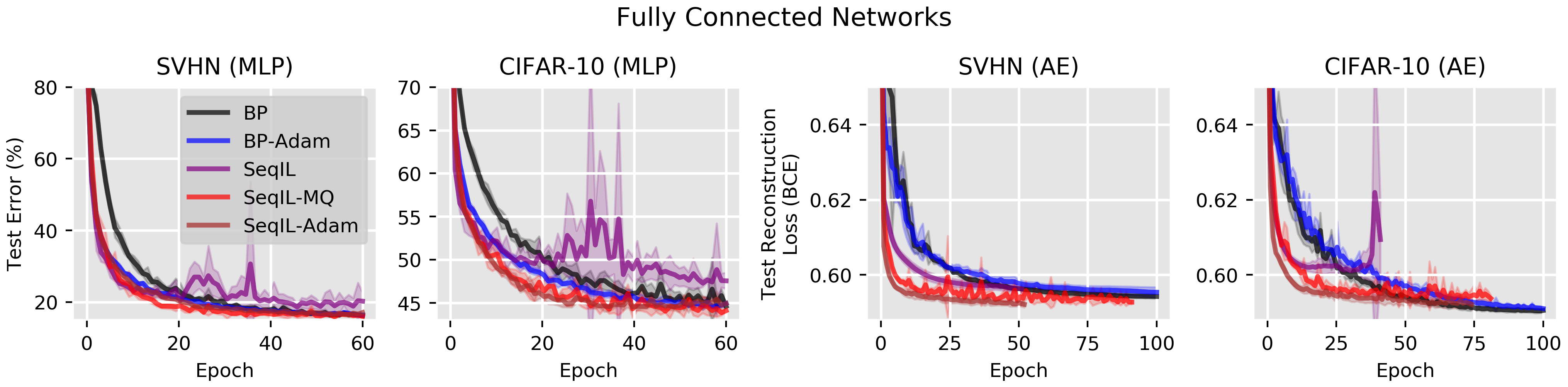}
\caption{Fully connected networks trained for classification tasks (MLP) and for autoencoder tasks (AE). Training runs are averaged over 5 seeds. Shaded error bars show standard deviation across seeds. The first 60 epochs of training shown for MLPs. For AEs, training runs are shown to convergence (lowest loss +1 epoch) or the max 100 epochs.}
\centering
\label{fig:FC}
\end{figure}

\begin{table}[t]
\centering
\begin{adjustbox}{width=1\textwidth}
\small
  \begin{tabular}{c | c c | c c | c c}
    \toprule
    \multicolumn{7}{c}{Classification Test Accuracies}\\
    \toprule
    & SVHN (FC) & SVHN (Conv) & CIFAR (FC) & CIFAR (Conv) & TinyImNet Top 1 & Top 5 \\
    \toprule
    BP & \textbf{84.40}$^{(\pm.12)}$ & $87.82^{(\pm.67)}$ & $56.55^{(\pm.08)}$ & \textbf{67.76}$^{(\pm.43)}$ & $23.14^{(\pm.42)}$ & $47.05^{(\pm.46)}$\\
    BP-Adam & $84.08^{(\pm.13)}$ & \textbf{89.02}$^{(\pm.19)}$ & $56.15^{(\pm.21)}$ & \textbf{67.51}$^{(\pm.45)}$ & \textbf{25.61}$^{(\pm.22)}$ & \textbf{50.04}$^{(\pm.41)}$\\
    SeqIL & $82.19^{(\pm.70)}$ & $ 81.55^{(\pm.59)}$ & $54.24^{(\pm.18)}$ & $50.67^{(\pm.76)}$ & $13.88^{(\pm.73)}$ & $33.42^{(\pm1.16)}$\\
    SeqIL-MQ & \textbf{84.51}$^{(\pm.17)}$ & $ 88.67^{(\pm.18)}$ & \textbf{57.70}$^{(\pm.21)}$ & \textbf{67.50}$^{(\pm.19)}$ & $23.91^{(\pm.31)}$ & $47.53^{(\pm.48)}$\\
    SeqIL-Adam & $84.13^{(\pm.15)}$ & $ 82.76^{(\pm.87)}$ & $56.21^{(\pm.20)}$ & $59.72^{(\pm.48)}$ & $18.22^{(\pm.29)}$ & $39.58^{(\pm.46)}$\\
    \bottomrule
\end{tabular}
\end{adjustbox}
\caption{Best test accuracies averaged across 5 seeds (mean $\pm$ standard dev.). Top scores and scores not significantly different from top score (according to two-sample t-test) are bolded.}
\vspace{-10pt}
\end{table}\label{tab:Acc}

\textbf{Autoencoders} We trained fully connected autoencoders (layer sizes 3072-1024-256-20-256-1024-3072) with ReLU at hidden layers and sigmoid at the output layer. Networks were trained on SVHN and CIFAR-10. A grid search was used to find a learning rate that produces the lowest test loss at convergence or after 100 epochs. Testing auto-encoders is important because, unlike classifiers, their output layers are large, which may make the task of propagating errors through the network easier for IL algorithms. IL models use T=6 inference iterations. We find that all IL models reduce the loss significantly faster than both BP and BP-Adam early in training (figure \ref{fig:FC}, right two plots). IL-MQ and IL-Adam converge to similar reconstruction losses as the BP algorithms.

\section{Related Work}
\textbf{Theoretical work on IL} Several recent works develop theoretical analyses of IL. Song et al. \cite{song2022inferring} argue IL and similar energy-based algorithms, like Constrastive Hebbian Learning \cite{movellan1991contrastive}, work differently than BP. Unlike BP, these algorithms first compute the desired neural activities, then update parameters to consolidate these neural activities (i.e., make them more probable). They call this approach 'prospective configuration'. They propose the brain likely learns in a way similar to prospective configuration and provide empirical support. Millidge et al. \cite{millidge2022theoretical} provide a formal analysis showing the neural activities computed by prospective configuration algorithms are similar to regularized Gauss-Newton targets, and IL is a variant EM algorithm known as maximum a posteriori learning. Alonso et al. \cite{alonso2022theoretical} developed an alternative, though compatible, analyses of IL that showed its approximation to the proximal algorithm/implicit SGD, and further examined the differences between IL and BP such as differences in stability across learning rates. The theoretical work here expands on these previous findings by showing 1) the results of Alonso et al. extend approximately to the more general case where a static learning rate is used at each layer (instead of a dynamic normalized learning rate) and 2) for the first time shows that, under the assumption IL is closely approximating implicit SGD, describes how IL is sensitive to higher order information. We believe our results provide a deeper understanding of how IL differs from BP/SGD than these previous works.

\textbf{Algorithmic Work on IL} The IL algorithm has recently been altered in several ways. One line of work altered IL in order to make its weight updates more similar to SGD (e.g., \cite{song2020can, millidge2020predictive}). Alonso et al. altered IL to make its weight updates better approximate the proximal algorithm/implicit SGD. Our alterations to IL differ from these previous works, as they do not seek to alter IL to make it more similar to SGD or implicit SGD. Salvatori et al. \cite{salvatori2022incremental} developed a variant of IL based on incremental EM \cite{neal1998view}, called incremental PC, which works by updating weights each inference iteration instead of at the end of the inference phase. Incremental PC has some benefits (e.g., more bio-plausible and reduces free energy quickly during inference). However, incremental PC still suffers from the delayed error propagation problem and its implementation in \cite{salvatori2022incremental} uses Adam optimizers. It is unclear whether incremental PC avoids convergence issues when Adam is not used. Incremental PC also requires more matrix multiplications each training iteration than SeqIL, given the same number of inference iterations (T) (see appendix table \ref{tab:iPCIL}). Our goal was to reduce computation of IL while improving convergence, so we do not further develop incremental PC here.

\section{Limitations}
More testing is needed to establish how well our methods work in larger networks with more than 3-5 hidden layers and to establish how our algorithms' behavior compares to BP across different batch sizes (we only test with mini-batch size 64). 

\section{Discussion}
Biologically motivated algorithms have historically faced at least two central challenges: 1) difficulty converging to test losses as good as BP/SGD when scaled to natural images \cite{bartunov2018assessing} and/or 2) requiring far more computation than BP/SGD. A common method for dealing with one or both of these issues is to make one's bio-plausible algorithm of choice more similar to BP/SGD (e.g., \cite{scellier2017equilibrium, song2020can, millidge2020predictive, laborieux2021scaling, ernoult2022towards}). In this paper, we took a different strategy. We took a highly biologically constrained algorithm, IL, and showed its differences from BP/SGD are not undesirable properties that ought to be removed. Instead, IL differs from BP/SGD because it implements a different optimization method, which has useful sensitivity to higher-order information. On the basis of our analysis, we altered the way its recurrent processing worked to significantly reduce computation and we created a custom optimizer made specifically to address the convergence issue of the IL algorithm. This is not unlike the way many standard optimizers, e.g., Adam, were specifically created to address shortcomings of BP/SGD. The result is an algorithm that, in our simulations on natural images, matched BP/SGD in terms of final test accuracy/loss, memory (table \ref{tab:MemComp}), and training time (in seconds) (fig. \ref{fig:RealTime}), while often reducing loss and converging more quickly than BP/SGD (in number of training iterations). These results suggest IL is a promising algorithm for further development and application in the bio-inspired and neuromorphic machine learning communities.

\bibliographystyle{plain}
\bibliography{neurips_2023}


\appendix

\section{Appendix}

\subsection{Derivation of Output Layer Update}\label{app:derivOut}
The prediction error at the output layer is and $e_L = \hat{h}_L - p_L$. Let $\mathcal{L}(y, \hat{h}_L) = \frac12 \Vert y - \hat{h}_L \Vert^2$, which implies
$-\frac{\partial \mathcal{L}(y, \hat{h}_L)}{\partial \hat{h}_L} = y - \hat{h}$, where $y$ is the output target.  To get the output layer update, we take the gradient of $-\frac{\partial F}{\partial \hat{h}_L}$ set to zero and solve for $\hat{h}_L$:
\begin{equation}\label{eq:outDerive}
\begin{split}
-\frac{\partial F}{\partial \hat{h}_L} = 0 &= -\gamma_L e_L -\frac{\partial \mathcal{L}(y, \hat{h}_L)}{\partial \hat{h}_L} \\
\hat{h}_L - y &= -\gamma_L (\hat{h}_L - p_L)  \\
(1 + \gamma_L) \hat{h}_L &= y + \gamma_L p_L \\
\hat{h}_L &= \frac{1}{(1 + \gamma_L)} y + \frac{\gamma_L}{(1 + \gamma_L)} p_L \\
\end{split}
\end{equation}

\subsection{Theorem 1 Proof, Lemmas, and Preliminaries}\label{app:thrm1proof}
\subsubsection{Preliminaries}
The proof of theorem \ref{thrm:impSGD=IL} rests on one key insight, which is that the parameters, $\theta$, after an IL update is performed at iteration $b$ can be defined as a function of neuron activities, i.e., $\theta = j(\hat{h})$, since the IL update is a function of neuron activities. In particular, we can define $\Delta \theta^{(b)} = g(\hat{h})$, where $\Delta \theta^{(b)}$ and $g(\hat{h})$ are the set of IL updates (equation \ref{eq:LMSUpdate}) at iteration $b$: 
\begin{equation}
\Delta \theta^{(b)} = g(\hat{h}) = [\Delta W_0, \Delta W_1,..., \Delta W_{L-1}] = [\alpha_0 e_{1} f(\hat{h}_0)^\top, \alpha_1 e_{2} f(\hat{h}_1)^\top,..., \alpha_{L-1} e_{L} f(\hat{h}_{L-1})^\top],
\end{equation}
where errors are just another function of $\hat{h}$: $e_l = \hat{h}_l - W_{l-1}f (\hat{h}_{l-1})$. Our updated parameters can then be defined as $\theta = j(\hat{h}) = \theta^{(b)} + g(\hat{h})$. Note that the range of function $j(\hat{h})$ covers all real valued vectors since $\hat{h}_l$ values may in principle be any real-valued vector, and therefore $g(\hat{h})$ and $j(\hat{h})$ output any real valued vector (see \cite{alonso2022theoretical} theorem 4.2).

This insight allows us to express the proximal update as an optimization procedure over hidden/auxiliary variable $\hat{h}$, instead of over parameters $\theta$. In particular, let $Prox(\theta) = \mathcal{L}(\theta) + \frac{1}{2\beta}\Vert\theta - \theta^{(b)} \Vert^2 = \mathcal{L}(\theta) + \frac{1}{2\beta}\Vert \Delta \theta \Vert^2$. We now have
\begin{equation}\label{implicitProx}
\begin{split}
Prox(\theta) = Prox(j(\hat{h}), \hat{h}) &= \mathcal{L}(j(\hat{h})) + \frac{1}{2\beta}\Vert j(\hat{h}) - \theta^{(b)} \Vert^2\\
&= \mathcal{L}(j(\hat{h})) + \frac{1}{2\beta}\Vert g(\hat{h}) \Vert^2.
\end{split}
\end{equation}
The proximal update now becomes
\begin{equation}\label{eq:implicitProxUpdate}
\theta^{(b+1)}_{prox} = \argmin_{j(\hat{h})} Prox(j(\hat{h}), \hat{h}) = j (\argmin_{\hat{h}} Prox(j(\hat{h}), \hat{h})),
\end{equation}
where the equality on the right follows from the fact $j(\hat{h})$ is a function optimized through changes to $\hat{h}$. The IL update can be rewritten using the same notation:
\begin{equation}\label{eq:implicitILUpdate}
\theta^{(b+1)}_{IL} = j (\argmin_{\hat{h}} F(\theta^{(b)}, \hat{h})).
\end{equation}

The proof for theorem 1 now follows from two lemmas:

\begin{lemma}
Under the assumptions of theorem \ref{thrm:impSGD=IL}, it is the case $\frac{\partial Prox((j(\hat{h}), \hat{h}))}{\partial \hat{h}} \propto \frac{\partial F(\theta^{(b)}, \hat{h})}{\partial \hat{h}}$.
\end{lemma}
The proof is shown in subsections \ref{app:lemma1P1} and \ref{app:lemma1P2}.

\begin{lemma}
If $\frac{\partial Prox((j(\hat{h}), \hat{h}))}{\partial \hat{h}} \propto \frac{\partial F(\theta^{(b)}, \hat{h})}{\partial \hat{h}}$, then $\argmin_{\hat{h}} Prox(j(\hat{h}), \hat{h}) = \argmin_{\hat{h}} F(\theta^{(b)}, \hat{h})$
\end{lemma}
\begin{proof}
If the gradients of $Prox$ and  $F$ w.r.t. $\hat{h}$ are proportional they will be zero for the same values of $\hat{h}$ and therefore share the same stationary points, including minima.
\end{proof}

\subsubsection{Proof of Theorem 1}
The proof for Theorem 1 is given as follows:
\begin{proof}
Lemma 1 and 2 imply $\argmin_{\hat{h}} Prox(j(\hat{h}), \hat{h}) = \argmin_{\hat{h}} F(\theta^{(b)}, \hat{h})$. It follows trivially from this that $j(\argmin_{\hat{h}} Prox(j(\hat{h}), \hat{h})) = j(\argmin_{\hat{h}} F(\theta^{(b)}, \hat{h}))$, and therefore $\theta^{(b+1)}_{prox} = \theta^{(b+1)}_{IL}$.
\end{proof}

\subsubsection{Proof Lemma 1 for Hidden Layer Updates}\label{app:lemma1P1}

\begin{proof}
 As with the minimization of F in PC networks, we consider the partial gradient of the terms of proximal objective using the terms that explicitly include $\hat{h}_l$. Since the loss term $\mathcal{L}(\hat{h})$ in $Prox(\theta)$ is only an explicit function of the output layer activities $\hat{h}_{L}$, we ignore the gradient of the loss with respect to hidden layers. This means that hidden layers only minimize the regularization term $\frac{1}{2\beta}\Vert \Delta \theta \Vert^2$. Given the notation from the previous section, we can express the minimization of the proximal regularization term as follows
\begin{equation}\label{eq:ProxReg}
\argmin_{\theta} \frac{1}{2 \beta}\Vert \Delta \theta \Vert^2 = \argmin_{j(\hat{h})} \frac{1}{2 \beta}\Vert g(\hat{h}) \Vert^2 = \argmin_{j(\hat{h})} \sum_{l=0}^{L-1} \frac{1}{2 \beta}\Vert \Delta W_l^{(b)} \Vert^2,
\end{equation}
where $\beta$ is a 'global' learning rate, not to be confused with matrix-wise step sizes $\alpha_l$. Let the function $Prox(\hat{h}_l$) be the components of equation \ref{eq:ProxReg} that explicitly involve $\hat{h}_l$, which is
\begin{equation}\label{eq:proxHid}
    Prox(\hat{h}_l) = \underbrace{\frac{1}{2\beta} \Vert \Delta W_{l-1} \Vert^2}_{prox_1} + \underbrace{\frac{1}{2\beta} \Vert \Delta W_{l} \Vert^2}_{prox_2}.
\end{equation}

Now we compute the gradient of $Prox(\hat{h}_l)$ w.r.t. $\hat{h}_l$ and show it is proportional to the gradient of F w.r.t. $\hat{h}_l$. The gradient can be expressed in terms of $prox_1$ and $prox_2$: $\frac{\partial Prox(\hat{h}_l)}{\partial \hat{h}_l} = \frac{\partial prox_1(\hat{h}_l)}{\partial \hat{h}_l} + \frac{\partial prox_2(\hat{h}_l)}{\partial \hat{h}_l}$. We first compute the gradient of $prox_1$.
\begin{equation}
    prox_1 = \frac{1}{2\beta}\Vert \Delta W_{l-1} \Vert^2 = \frac{1}{2\beta} \Vert  \alpha_{l-1} e_l f(\hat{h}_{l-1})^T \Vert^2,
\end{equation}

where $p_l = W_{l-1} f(\hat{h}_{l-1})$ and $e_l = \hat{h}_l - p_{l}$. Using the chain rule 

\begin{equation}
    \frac{\partial prox_1(\hat{h}_l)}{\partial \hat{h}_l} = \frac{\partial prox_1}{\partial \Delta W_{l-1}} \frac{\partial \Delta W_{l-1}}{\partial e_{l}} \frac{\partial e_{l}}{\partial \hat{h}_{l}},
\end{equation}

where $\frac{\partial prox_1}{\partial \Delta W_{l-1}} = \frac{\alpha_{l-1}}{\beta} e_l f(\hat{h}_{l-1})^T$ and $\frac{\partial \Delta W_{l-1}}{\partial e_l} \frac{\partial e_{l}}{\partial \hat{h}_{l}} = \alpha_{l-1} f(\hat{h}_{l-1})$. Multiplying together we get

\begin{equation}
    \frac{\partial prox_1(\hat{h}_l)}{\partial \hat{h}_l} = \frac{\alpha_{l-1}^2}{\beta} \Vert f(\hat{h}_{l-1}) \Vert^2 e_l.
\end{equation}

Now we derive $\frac{prox_2(\hat{h}_l)}{\hat{h}_l}$:

\begin{equation}
    prox_2 = \frac{1}{2\beta} \Vert \Delta W_{l} \Vert^2 = \frac{1}{2\beta} \Vert  \alpha_{l} e_{l+1}f(\hat{h}_{l})^T \Vert^2,
\end{equation}

where $p_{l+1} = W_{l} f(\hat{h}_{l})$ and $e_{l+1} = \hat{h}_{l+1} - p_{l+1}$. Using the chain rule

\begin{equation}
    \frac{\partial prox_2}{\partial \hat{h}_l} = \frac{\partial prox_2}{\partial \Delta W_{l}} \frac{\partial \Delta W_{l}}{\partial e_{l+1}} \frac{\partial e_{l+1}}{\partial p_{l+1}} \frac{\partial p_l}{\partial \hat{h}_{l}}  + \frac{\partial prox_2}{\partial \Delta W_{l}} \frac{\partial \Delta W_{l}}{\partial f(\hat{h}_{l})^T} \frac{\partial f(\hat{h}_l)^T}{\partial \hat{h}_{l}}.
\end{equation}

The term on the left is composed of the following gradients: $\frac{\partial prox_2}{\partial \Delta W_{l}} =  \frac{\alpha_{l}}{\beta} e_{l+1} f(\hat{h})^T_{l},  \frac{\partial \Delta W_{l}}{\partial e_{l+1}} = \alpha_{l} f(\hat{h}_l)$, and $\frac{\partial e_{l+1}}{\partial p_{l+1}}\frac{\partial p_l}{\partial \hat{h}_{l}}= -f'(\hat{h}_l)W_l^T$. The term on the right comes out to $\frac{\partial prox_2}{\partial \Delta W_{l}} \frac{\partial \Delta W_{l}}{\partial f(\hat{h})^T_{l}}\frac{\partial f(\hat{h}_l)^T}{\partial f(\hat{h})_l} = \alpha_{l-1} (e_{l+1}^T e_{l+1}) f'(\hat{h}_l) \hat{h}_l$, where $f'(\hat{h}_l)$ is gradient of $f$. Substituting these terms back into the original equation, we get

\begin{equation}
\begin{split}
    \frac{\partial prox_2}{\partial \hat{h}_l} &= -\frac{\alpha_{l}^2}{\beta} f'(\hat{h}_l) W_l^T e_{l+1} f(\hat{h}_l)^T f(\hat{h}_l) + \frac{\alpha_{l}^2}{\beta} (e_{l+1}^T e_{l+1}) f'(\hat{h}_l) \hat{h}_l\\
    &= - \frac{\alpha_{l}^2}{\beta} \Vert f(\hat{h}_l) \Vert^{2} f'(\hat{h}_l) W_l^T e_{l+1} + \frac{\alpha_{l}^2}{\beta} \Vert e_{l+1} \Vert^2 f'(\hat{h}_l) \hat{h}_l.
\end{split}
\end{equation}

Now we combine the gradients for $prox1$ and $prox2$ and use the assumptions of the theorem, which are $\gamma_{l+1} = \alpha_{l}^2 \Vert f(\hat{h}_l) \Vert^{2}$, $\gamma_l = \alpha^2_{l-1} \Vert f(\hat{h}_{l-1}) \Vert^2$, and $\gamma^{decay}_l = \alpha_{l}^2 \Vert e_{l+1} \Vert^2$:
\begin{equation}
\begin{split}
\frac{\partial Prox(\hat{h}_l)}{\partial \hat{h}_l} &= - \frac{\alpha_{l}^2}{\beta} \Vert f(\hat{h}_l) \Vert^{2} f'(\hat{h}_l) W_l^T e_{l+1} + \frac{\alpha^2_{l-1}}{\beta} \Vert f(\hat{h}_{l-1}) \Vert^2 e_l + \frac{\alpha_{l}^2}{\beta} \Vert e_{l+1} \Vert^2 f'(\hat{h}_l) \hat{h}_l\\
&\propto - \alpha_{l}^2 \Vert f(\hat{h}_l) \Vert^{2} f'(\hat{h}_l) W_l^T e_{l+1} + \alpha^2_{l-1} \Vert f(\hat{h}_{l-1}) \Vert^2 e_l + \alpha_{l}^2 \Vert e_{l+1} \Vert^2 f'(\hat{h}_l) \hat{h}_l\\
&= - \gamma_{l+1} f'(\hat{h}_l) W_l^T e_{l+1} + \gamma_l e_l + \gamma^{decay}_l f'(\hat{h}_l) \hat{h}_l\\
&= \frac{\partial F}{\partial \hat{h}_l}
\end{split}
\end{equation}
\end{proof}

\subsubsection{Proof of Lemma 1 for Output Layer Updates}\label{app:lemma1P2}
\begin{proof}
Let $\hat{h}_L^{IL}$ be the output layer value updated according to equation \ref{eq:outUpdate} and $\hat{h}_L^{prox}$ be the output layer activity updated according to the gradients of the proximal objective. Here we show $\hat{h}_L^{IL} = \hat{h}_L^{prox}$ in the limit where $\hat{h}_{L-1} \rightarrow h_{L-1}^{(b+1)}$ and $\gamma_L = \alpha_{L-1} ( \frac{1}{\beta} + \Vert f(\hat{h}_{L-1})\Vert^2) - 1$. 

The terms in $Prox(\theta)$ that explicitly contain $\hat{h}_L$ are $\mathcal{L}(j(\hat{h}, \hat{h}))$ and $\frac{1}{2\beta} \Vert \Delta W_{L-1} \Vert^2$. We know from the previous section that
\begin{equation}
    \frac{\partial \frac{1}{2\beta} \Vert \Delta W_{L-1} \Vert^2}{\partial \hat{h}_L} = \frac{\alpha_{L-1}^2}{\beta} \Vert f(\hat{h}_{L-1}) \Vert^2 e_L = \frac{c\alpha_{L-1}}{\beta} e_L,
\end{equation}
where $c = \alpha_{L-1} \Vert f(\hat{h}_{L-1}^{(b)})\Vert^2$. The loss $\mathcal{L}(j(\hat{h}), \hat{h}) = \frac12 \Vert y^{(b)} - h_L^{(b+1)} \Vert^2$, where $h_L^{(b+1)}$ are the feed-forward activities after the weight updates are applied. We first derive an expression for $h_L^{(b+1)}$ in the limit where $\hat{h}_{L-1} \rightarrow h_{L-1}^{(b+1)}$ then compute the gradient of the loss.
\begin{equation}
\begin{split}
\lim_{\hat{h}_{L-1} \rightarrow h_{L-1}^{(b+1)}} h^{(b+1)}_{L} = W^{(b+1)}_{L-1} f(\hat{h}_{L-1}^{(b)}) &= (W^{(b)}_{L-1} + \Delta W^{(b)}_{L-1}) f(\hat{h}_{L-1}^{(b)})\\
&= W^{(b)}_{L-1}f(\hat{h}_{L-1}^{(b)}) + \alpha_{L-1} f(\hat{h}_{L-1}^{(b)})^{\top} f(\hat{h}_{L-1}^{(b)}) e^{(b)}_{L}\\
&= W^{(b)}_{L-1}f(\hat{h}_{L-1}^{(b)}) + c \hat{h}_L^{(b)} - c W^{(b)}_{L-1}f(\hat{h}_{L-1}^{(b)})\\
&= (1 - c) W^{(b)}_{L-1}f(\hat{h}_{L-1}^{(b)}) + c \hat{h}_L^{(b)}\\
&= (1 - c) p^{(b)}_{L} + c \hat{h}_L^{(b)},
\end{split}
\end{equation}
where $c = \alpha_{L-1} \Vert f(\hat{h}_{L-1}^{(b)})\Vert^2$. The second line follows from the assumption $\hat{h}_{L-1} \rightarrow h_{L-1}^{(b+1)}$. Now we can compute the gradient of the loss:
\begin{equation}
\begin{split}
\frac{\partial \mathcal{L}(y^{(b)}, h^{(b+1)}_L)}{\partial \hat{h}_{L}^{(b)}} &= \frac{\partial \mathcal{L}(y^{(b)}, h^{(b+1)}_L)}{\partial h^{(b+1)}} \frac{\partial h^{(b+1)}_L}{\partial \hat{h}_{L}^{(b)}}\\
&= c \frac{\partial \mathcal{L}(y^{(b)}, h^{(b+1)}_L)}{\partial h^{(b+1)}}\\
&= c (((1 - c) p_L^{(b)} + c\hat{h}_L^{(b)}) - y)\\
\end{split}
\end{equation}

Under the assumption $\gamma_L  = \frac{\alpha_L}{\beta} + c - 1 = \alpha_{L-1} ( \frac{1}{\beta} + \Vert f(\hat{h}_{L-1})\Vert^2) - 1$ we have
\begin{equation}
\begin{split}
\frac{\partial Prox(\hat{h}_L)}{\partial \hat{h}_L} &= \frac{c \alpha_{L-1}}{\beta} e_L + c (((1 - c) p_L^{(b)} + c\hat{h}_L^{(b)}) - y)\\
&\propto \frac{\alpha_{L-1}}{\beta}  e_L + (1 - c) p_L^{(b)} + c\hat{h}_L^{(b)} - y\\
&= \frac{\alpha_{L-1}}{\beta}  \hat{h}_L - \frac{\alpha_{L-1}}{\beta} p_L + (1 - c) p_L^{(b)} + c\hat{h}_L^{(b)} - y\\
&= (\frac{\alpha_{L-1}}{\beta} - c)  \hat{h}_L - (\frac{\alpha_{L-1}}{\beta} + c - 1) p_L^{(b)} - y\\
&= (\frac{\alpha_{L-1}}{\beta} - c - 1)  (\hat{h}_L - p_L^{(b)}) + \hat{h}_L - y\\
&= \gamma_L e_L + \frac{\partial \mathcal{L}(y^{(b)}, \hat{h}_L)}{\partial \hat{h}_L}\\
&= \frac{\partial F}{\partial \hat{h}_L}.
\end{split}
\end{equation}

\end{proof}

\subsection{Extending to Mini-batches}\label{app:minib}
Theorem \ref{thrm:impSGD=IL} considers the case where mini-batch size 1 is used. We chose this condition because 1) it simplifies the mathematics of the proof, 2) it more realistically describes the learning scenario the brain faces (see appendix of \cite{alonso2022theoretical}), and 3) because the extension of theorem 1 to larger mini-batches is straight-forward. In particular, let $\Delta \theta_{prox}^{(b)} = \theta^{(b+1), n}_{prox} - \theta^{(b)}$, be the proximal/implicit gradient update for the n-th data point in the mini-batch at iteration $b$. As long as the same $\gamma$ setting and assumptions about the learning rate in theorem 1 hold for each data point in the mini-batch, it follows trivially that
\begin{equation}
\frac{1}{N} \sum_{n=0}^N\Delta \theta_{IL}^{(b),n} = \frac{1}{N}\sum_{n=0}^N\Delta \theta_{prox}^{(b),n}.
\end{equation}
As long as the assumption of theorem 1 holds for each individual data point in the mini-batch, then mini-batching with IL is equivalent to computing the proximal update/implicit gradient for each individual data point in the mini-batch independently and then averaging over the updates/gradients. Thus, IL is equivalent to a kind of stochastic version of the proximal algorithm, where the proximal minimization problem is solved over individual data-points (then the solutions are averaged) rather than over entire (mini-)batches. This is why we claim IL approximates \textit{stochastic} implicit gradient descent \cite{toulis2014implicit, toulis2016stochastic}, rather than the batch/non-stochastic version.

Theorem \ref{thrm:2ndOrder} extends to the mini-batch case in a similar way. In particular, for small $\beta$ values, IL would approximate the process of computing the regularized Newton-Raphson update for each individual data point in the mini-batch and averaging those updates together.

\subsection{Proof of Theorem 2}\label{app:proofThrm2}
\begin{proof}

To prove this theorem we perform a Taylor expansion of the gradient around point $\theta^{(b)}$ to get an approximation of the implicit gradient. In particular, we approximate the function $f(\theta^{(b+1)}) = f(\theta^{(b)} + \Delta \theta_{imp}) =  -\frac{\partial \mathcal{L}(\theta^{(b+1)})}{\partial \theta^{(b+1)}}$, where $\Delta  \theta_{imp}$ is the change to parameters after an implicit SGD/proximal update, which is just $- \beta \frac{\partial \mathcal{L}(\theta^{(b+1)})}{\partial \theta^{(b+1)}}$. Let the current training iteration be iteration $b$ and $H = \frac{\partial^2 \mathcal{L}(\theta^{(b)})}{\partial \theta^{(b)2}}$. We have
\begin{equation}
\begin{split}
f(\theta^{(b)} + \Delta \theta_{imp}) &= f(\theta^{(b)}) + \frac{\partial f(\theta^{(b)})}{\partial \theta^{(b)}} \Delta\theta_{imp} + \mathcal{O}(\beta^2)\\ 
-\frac{\partial \mathcal{L}(\theta^{(b+1)})}{\partial \theta^{(b+1)}} &= - \frac{\partial \mathcal{L}(\theta^{(b)})}{\partial \theta^{(b)}} - H \beta \frac{\partial \mathcal{L}(\theta^{(b+1)})}{\partial \theta^{(b+1)}} + \mathcal{O}(\beta^2)\\
-\frac{\partial \mathcal{L}(\theta^{(b+1)})}{\partial \theta^{(b+1)}}  + \beta H \frac{\partial \mathcal{L}(\theta^{(b+1)})}{\partial \theta^{(b+1)}} &= -  \frac{\partial \mathcal{L}(\theta^{(b)})}{\partial \theta^{(b)}} + \mathcal{O}(\beta^2)\\
- (I + \beta H) \frac{\partial \mathcal{L}(\theta^{(b+1)})}{\partial \theta^{(b+1)}} &= - \frac{\partial \mathcal{L}(\theta^{(b)})}{\partial \theta^{(b)}} + \mathcal{O}(\beta^2)\\
- \frac{\partial \mathcal{L}(\theta^{(b+1)})}{\partial \theta^{(b+1)}} &= - (I + \beta H)^{-1} \frac{\partial \mathcal{L}(\theta^{(b)})}{\partial \theta^{(b)}} + \mathcal{O}(\beta^2)\\
\end{split}
\end{equation}
From this and the assumption $\Delta \theta^{(b)}_{IL} = -\beta \frac{\partial \mathcal{L}(\theta^{(b+1)})}{\partial \theta^{(b+1)}}$, it follows that $\Delta \theta^{(b)}_{IL} = - \beta \frac{\partial \mathcal{L}(\theta^{(b+1)})}{\partial \theta^{(b+1)}} \approx - (I + \beta H)^{-1} \beta \frac{\partial \mathcal{L}(\theta^{(b)})}{\partial \theta^{(b)}}$, with error $\mathcal{O}(\beta^2)$.
\end{proof}

\subsection{Methods}\label{app:methods}
Below we outline details of each experiment with the intent of making the experiments replicable. Code for all experiments can be found at \url{https://github.com/nalonso2/PredictiveCoding-MQSeqIL/tree/main}.  All simulations were run on the same laptop with a small GPU type NVIDIA GeForce RTX 270 with MAXQ design.

\begin{figure}[h]
\includegraphics[width=.98\textwidth]{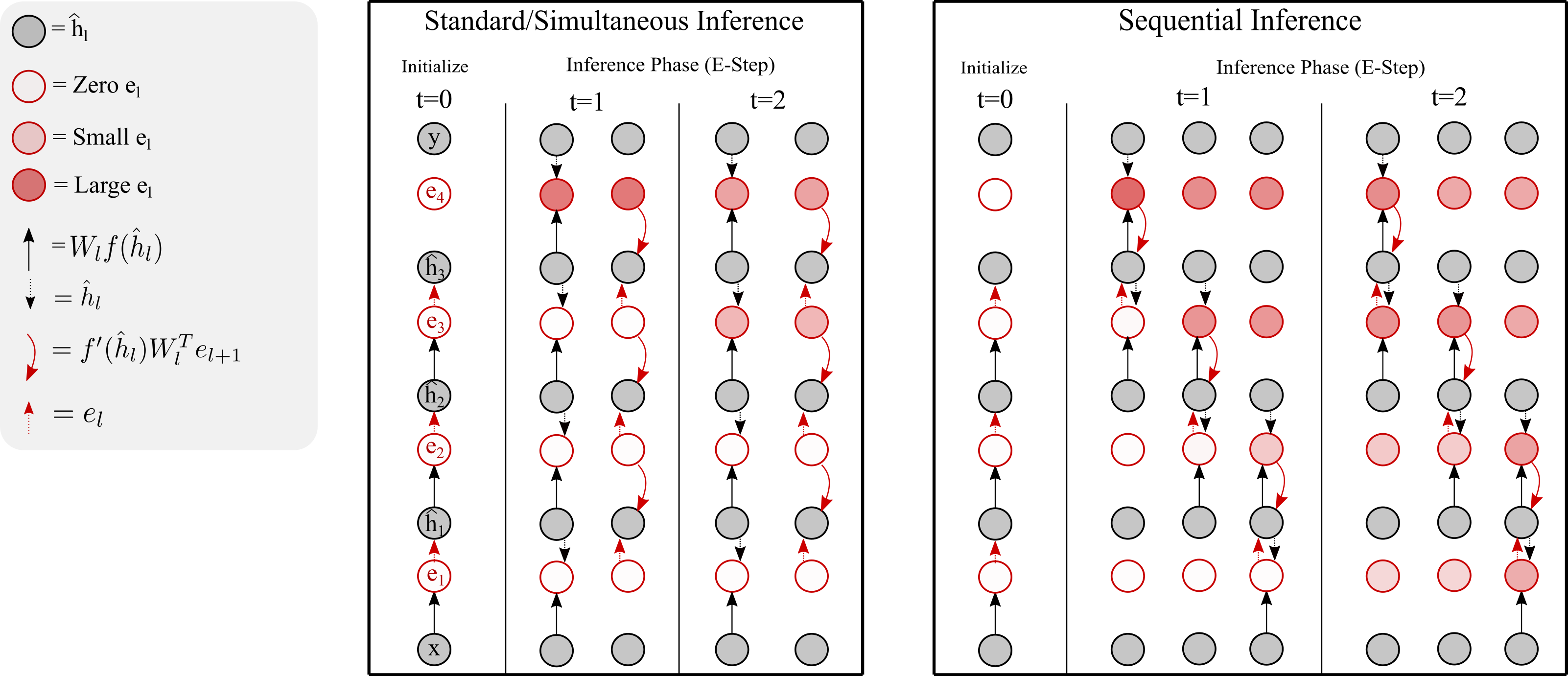}
\caption{Description of the standard/simultaneous inference versus sequential inference method. Initialization (t=0) and the first two inference steps (t=1 and t=2) are shown. Gray nodes are activities $\hat{h}_l$. Red nodes are errors $e_l$. At initialization input layer is clamped to $x$, and we assume the fully clamped scenario where the output layer, top nodes, are clamped to the target: $\hat{h}_L = y$. Hidden layers are initialized to their feed forward activities: $\hat{h}_l = h_l$. Each inference iteration, simultaneous inference begins by computing and storing errors at all layers (left column under t=1 and t=2), then uses the stored errors to update activities (right column under t=1 and t=2). We can see that this leads to a delay in the way errors are propagated through the network. Sequential inference begins by computing errors $e_4$ and $e_3$, then uses those to update $\hat{h}_3$ (left column under t=1 and t=2). Then it recomputes error $e_3$, which will now be non-zero, and $e_2$, then uses those to update $\hat{h}_2$. This process repeats until all activities are updated. We can see sequential inference propagates errors through the entire network in a single iteration.}
\centering
\label{fig:simVseqDiagram}
\end{figure}

\textbf{Details on Implementation of IL} High level summaries of sequential and standard IL are given in algorithms \ref{alg:1} and \ref{alg:2} and in figure \ref{fig:simVseqDiagram}. We provide a more detailed description of the implementation of IL and SeqIL in algorithms \ref{alg:ILFull} and \ref{alg:SeqILFull}. In the simulations and figure, we fully clamp the output layer and set $\gamma_l = 1$ and $\gamma_l^{decay} = 0$, which is common in practice.
\begin{minipage}{0.48\textwidth}
\begin{algorithm}[H]
\SetAlgoLined
\DontPrintSemicolon
\Begin{
    \tcp{Initialize: $\hat{h} = h$}
    \For{$t=1$ \KwTo $T$}{
        \tcp{Compute and store errors at layers 1 to L}
        \tcp{Update $\hat{h}$ at layers 1 to L using stored errors}
    }
}
\caption{Simultaneous Inference}
\label{alg:1}
\end{algorithm}
\end{minipage}
\hfill
\begin{minipage}{0.49\textwidth}
\begin{algorithm}[H]
\SetAlgoLined
\DontPrintSemicolon
\Begin{
\tcp{Initialize: $\hat{h} = h$}
    \For{$t=1$ \KwTo $T$}{
        \For{reversed($l=1$ \KwTo $L$)}{
            \tcp{Compute $e_l$ and $e_{l+1}$}
            \tcp{Update $\hat{h}_l$}
            }
    }
}
\caption{Sequential Inference \label{alg:2}}
\end{algorithm}
\end{minipage}

\begin{algorithm}[h]
\SetAlgoLined
\DontPrintSemicolon
\Begin{
    \tcp{Initialize}
    $\hat{h}_0 \leftarrow x^{(b)}$\;
    \For{$l=0$ \KwTo $L-1$}{
        $\hat{h}_{l+1} \leftarrow W_l f(\hat{h}_l)$
    }
    $\hat{h}_L \leftarrow y^{(b)}$\\
    \tcp{Inference Phase}
    \For{$t=0$ \KwTo T}{
        $e_{L} \leftarrow \hat{h}_L - \sigma(W_{L-1}f(\hat{h}_{L-1}))$\\
        \For{$l=1$ \KwTo $L-1$}{
            $e_{l} \leftarrow \hat{h}_l - W_{l-1}f(\hat{h}_{l-1})$
        }
        \For{$l=1$ \KwTo $L$}{
            $\hat{h}_{l} \leftarrow \hat{h}_{l} - \epsilon (e_l - f'(\hat{h_l})W_{l}^T e_{l+1})$
        }
    }
    \tcp{Update Weight Matrices}
        Equation \ref{eq:LMSUpdate}\;
}
\caption{IL ($\gamma_l = 1$,$\gamma^{decay}_l = 0$, $\sigma=$softmax or sigmoid)}\label{alg:ILFull}
\end{algorithm}
\begin{algorithm}[h]
\SetAlgoLined
\DontPrintSemicolon
\Begin{
    \tcp{Initialize}
    $\hat{h}_0 \leftarrow x^{(b)}$\;
    \For{$l=0$ \KwTo $L-1$}{
        $\hat{h}_{l+1} \leftarrow W_l f(\hat{h}_l)$
    }
    $\hat{h}_L \leftarrow y^{(b)}$\\
    \tcp{Inference Phase}
    \For{$t=0$ \KwTo T}{
        $p \leftarrow \sigma( W_{L-1}f(\hat{h}_{L-1}))$\\
        \For{reversed($l=1$ \KwTo $L$)}{
            $e_{l+1} \leftarrow \hat{h}_{l+1} - p$\\
            $p = W_{l-1}f(\hat{h}_{l-1})$\\
            $e_{l} \leftarrow \hat{h}_l - p$\\
            $\hat{h}_{l} \leftarrow \hat{h}_{l} - \epsilon (e_l - f'(\hat{h_l})W_{l}^T e_{l+1})$
        }
    }
    \tcp{Update Weight Matrices}
        Equation \ref{eq:LMSUpdate}\;
}
\caption{SeqIL ($\gamma_l = 1$,$\gamma^{decay}_l = 0$, $\sigma=$softmax or sigmoid))}\label{alg:SeqILFull}
\end{algorithm}
\textbf{Details on MQ Optimizer Implementation}\label{app:MQMethod} In our tests of the MQ optimizer, we use hyperparameter settings $\alpha_{min} = .001$, $r = .000001$, and $\rho = .9999$ for fully connected and $\rho =.999$ for convolutional networks. These hyper-parameter settings perform well generally across models, and using the same hyper-parameter settings across simulations significantly reduced hyper-parameter search. We only needed to then perform grid searches to find learning rates, $\alpha_l$. We find initializing $v_l = \alpha_l$ works well. The variable $v_l$, which controls the adaptive learning rate in MQ is a slow moving average of the weight update magnitude:
\begin{equation}
v_l^{(b+1)} = (1 - \rho) v_l^{(b)} - \rho \frac{1}{ij}\vert \frac{\partial F}{\partial W_{l}^{(b)}} \vert,
\end{equation}
where $\vert \frac{\partial F}{\partial W_{l}^{(b)}} \vert$ is the L-1 norm and i and j are the row and column sizes respectively. This moving average will be biased early in training toward the initial value of $v_l$, making $v_l$ and MQ highly sensitive to initialization. To prevent this bias we use a modified $\rho^*$ to more quickly update $v_l$ early in training:
\begin{equation}
\rho^* = \text{min}(\rho, 1/(b+2)),
\end{equation}
where $b$ is the training iteration. $v_l$ is then updated using $\rho^*$ instead of $\rho$. After a finite number of iterations $\rho^*=\rho$, but early in training $\rho^* > \rho$ allowing for faster updating of $v_l$. We find initializing $v_l = \alpha_l$ works well.

\textbf{Comparison of Standard and Sequential Inference (Figure \ref{fig:seqInfTest})} We measure the mean of the squared errors at each layer of a MLP with three hidden layers (3072-3x1024-10). ReLUs are used at hidden layers and softmax at output layer. The same step size of .02 is used to update layer activities in both models. Weights are randomly initialized so training cannot skew the results. CIFAR-10 dataset is used. Results are averaged over 10 seeds. For each seed we test on 10000 images. MSE is computed each inference iteration and averaged across images, then these averaged runs are averaged across seeds.

\textbf{Weight Update Analysis (Figure \ref{fig:wtAnlyze} and \ref{fig:wtAnlyzConv})} We train MLPs with three hidden layers, dimensions 3072-3x1024-10, on CIFAR-10. Five seeds are trained for each kind of algorithm. We train using standard IL (IL), standard IL with MQ (IL-MQ), and standard IL with Adam (IL-Adam). ReLU activations are used at hidden layers and softmax at the output layer. We train the network for 10000 iteration. After 1000 iterations we begin measuring the average update magnitude (we have this delay to allow the adaptive learning rates in Adam and MQ optimizers to 'warm up'). The average update magnitude is computed $\frac{1}{IJ}\sum_{i} \sum_j \vert \Delta W_{l, ij}\vert$, where $i$ and $j$ are the row and column indexes and $I$ and $J$ are the size of row and columns, respectively, of weight matrix $l$. The update magnitudes for each matrix are averaged over training iterations and seeds to obtain the final averages shown in figure \ref{fig:wtAnlyze}. This same procedure is used to compute the weight updates for figure \ref{fig:wtAnlyzConv}.

\textbf{Proximal Objective Tests (Figure \ref{fig:proxPart} and \ref{fig:proxFull})}\label{app:proxMethod}
Fully connected networks with 3 hidden layers and ReLU activations are first trained for 1000 iterations on CIFAR-10 then the proximal objective value is measured during inference over 10000 data points and averaged. These averaged runs are then averaged across 5 seeds to get the final values.

For the tests in figure \ref{fig:proxPart} and \ref{fig:proxFull}, we test the IL algorithms under the typical settings of $\gamma$ (rather than those in theorem 1), where $\gamma_n = 1$ for all $n$, and $\gamma_n^{decay} = 0$ for all $n$. We then measure the proximal objective during inference under the $\beta$ values $[.01, .1, 1, 10, 100]$. How the proximal objective changes during inference under each of these $\beta$ values is shown in figure \ref{fig:proxFull}. How it changes under the values $\beta=1$ is shown in figure \ref{fig:proxPart} (with the exception of IL-MQ which is shown with $\beta=100$, since this is the value it best minizes).

The proximal objective is measured each inference iteration, $t$, as follows: model parameters are updated each inference iteration using the IL learning rule and associated learning rate. The loss is then computed after the update (the $\mathcal{L}$ in the proximal objective) and the squared l-2 norm of the parameter update (the regularization term in the proximal objective) is computed. Parameter updates are then erased, another inference step is performed and the proximal objective value is recomputed.

\textbf{Classification with IL versus SeqIL (Figure \ref{fig:seqVsimAcc})}
We trained a four layer MLP, dimensions 3072-4x1024-10 with standard IL and sequential IL using Adam optimizers, since Adam has becme the standard optimizer used in previous works. Models are trained on CIFAR-10, with ReLU at hidden layers, softmax output layer.  Grid searches were used to find the learning rate and step size for activity updates, $\epsilon$, for each value of $T$. We searched over learning rates $[.0001, .00005, .00001]$ and step sizes for the activity updates $[.02, .05, .1, .3]$ for each value of $T$. After the grid search we trained 5 seeds at each value of $T$ and show the mean and standard deviation over the best accuracies achieved by each seed. 

\textbf{Classification and Autoencoder Training (Figures \ref{fig:seqVsimAcc}-\ref{fig:FC})}\label{app:methodTraining}
For classification tasks on SVHN, CIFAR-10, and Tiny Imagenet, we train fully connected MLPs dimension 3072-3x1024-10 on SVHN and CIFAR-10, and convolutional networks on SVHN, CIFAR-10, and Tiny Imagenet, with ReLU activitation at hidden layers and softmax at output layers. Due to limited compute, we used small convolutional networks, similar to previous work on biologically inspired algorithms \cite{bartunov2018assessing}, where max pools are removed and filters with stride of 2 are used to reduce the number of channels each layer. SVHN and CIFAR-10 use three convolutional layers (5x5-64, 5x5-128, 3x3-256) followed by one fully connected layer. Tiny imagenet used a network with four layers (5x5-32, 5x5-64, 5x5-128, 3x3-256) with one fully connected layer. All of the IL algorithms used a highly truncated inference phases, where T=3. In addition to sequential IL with MQ (SeqIL-MQ), we test SeqIL with no optimizer (SeqIL) and SeqIL with Adam (SeqIL-Adam). We compare to the performance of BP-SGD and BP with Adam (BP-Adam). Learning rates were first hand tuned to find ranges that yield good performance, then grid searches over 5-10 learning rates were performed to find the learning rate that the best test performance at convergence. The same approach for finding learning rates is used for the auto-encoder tasks. Autoencoders are fully connected networks with layers sizes 3072-1024-256-20-256-1024-3072, with ReLU at hidden layers and sigmoid at the output layer. 

\textbf{Real Time Training (Figure \ref{fig:RealTime})}
To get the training time in second for models trained on CIFAR-10, we trained a model of each type (FC classifier, convolutional classifier, and FC autoencoder) on CIFAR-10 without testing and without storing any values, so that testing and changes in available RAM from storing training data, do not affect training time. All models were trained on the same laptop with a small GPU type NVIDIA GeForce RTX 270 with MAXQ design. Built in python timer functions were used to get the training time for one training run for each model type. Training times were then plotted against the test losses, which were computed from the simulations described in the previous section.

\textbf{Determining Computation Requirements (Table \ref{tab:BPILComp})} Although matrix multiplies do not account for all of the operations performed, it is the most numerous operation in BP and IL algorithms, so it is a good indicator of computation cost for each algorithm. Each training iteration, BP must compute a feed forward signal, which requires L matrix multiples, where L is the number of matrices in the network. To propagate feedback signals, L-1 matrix multiplies are needed, and to perform weight updates another L matrix multiplies are needed, totaling 3L-1. IL algorithms similarly perform L matrix multiplies to compute the predictions (forward signals) and errors, and another L-1 to update activities, totaling 2L-1 matrix multiplies each inference iteration, totaling T(2L-1) for T inference iterations. Then another L matrix multiplies are needed to update weights, totaling T(2L-1) + L matrix multiplies.

\subsection{Supplementary Results}

\begin{figure}[h]
\includegraphics[width=\textwidth]{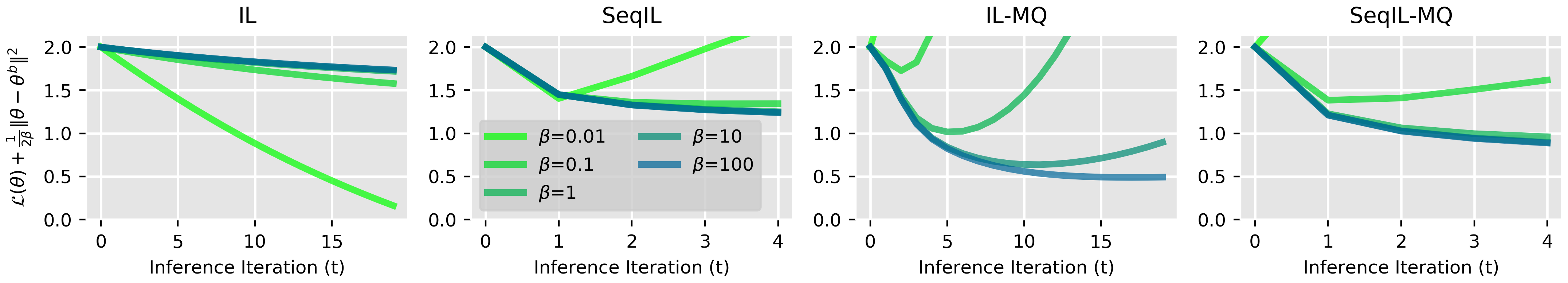}
\caption{Measurement of the proximal objective during the inference phase of IL in a fully connected networks trained on CIFAR-10. Algorithms are tested under settings $\gamma_l = 1$ and $\gamma_l^{decay} = 0$ for all $l$, which is the standard in previous work (e.g., \cite{rao1999predictive, whittington2017approximation, alonso2021tightening, salvatori2021associative, alonso2022theoretical}). We test IL with a fixed learning rate (IL), SeqIL with a fixed learning rate (SeqIL), and these algorithms with the MQ optimizer. The proximal objective is measured under different values of the $\beta$ term. For easier comparison each inference run is shifted so its starting value is 2. For a wide range of $\beta$ values the proximal objective is reduced during inference. \textit{According to theorem \ref{thrm:impSGD=IL}, performing inference with a fully clamped output layer, as we do here, approximates minimizing the proximal objective under a large $\beta$ value. All models significantly reduce the proximal objective for the largest $\beta$ value we tested of 100. Reduction of the proximal objective is less reliable for very small $\beta$ values. These results are therefore highly consistent with theorem \ref{thrm:impSGD=IL}}.}
\centering
\label{fig:proxFull}
\end{figure}
\begin{figure}[h]
\includegraphics[width=.47\textwidth]{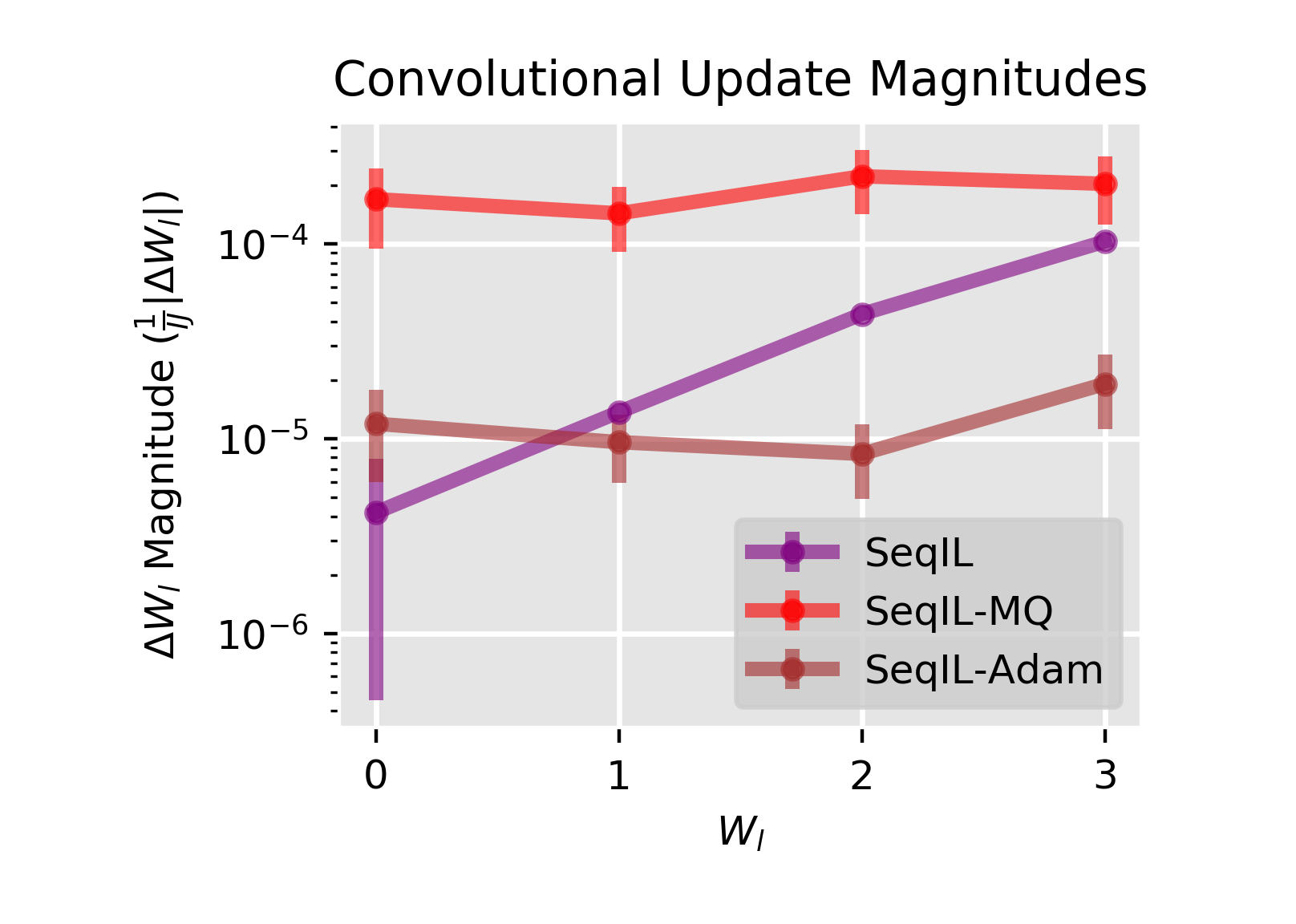}
\centering
\caption{Average update magnitude of each weight matrix/kernal in convolutional network with three hidden layers trained on CIFAR-10. The same network that was trained above was used: three convolutional layers (5x5-64, 5x5-128, 3x3-256), stride 2, followed by one fully connected layer, ReLUs at hidden layers, softmax output. The learning rates that achieved best test accuracy at convergence are used. In this convolutional network, we find that SeqIL-Adam equalized weight update magnitudes well, but tends to produce very small weight updates overall compared to SeqIL-MQ. This is different than the fully connected case, where the two produce very similar update magnitudes (figure \ref{fig:wtAnlyze}). Weight update magnitudes can be increased by increasing the learning rate parameter in the Adam optimizer, but we find learning quickly becomes unstable when we do this. The fact that weight updates are so small may help explain why SeqIL-Adam falls into shallow local minima, i.e., because it is unable to push parameters far enough from the initial values to find deep regions of the loss landscape, leading parameters to get caught in shallow local minima near the initial parameter values. SeqIL-MQ, on the other hand, equalizes weight update magnitudes but produces much larger updates on average which likely helps parameters find deeper minima. The reason SeqIL-MQ yields large weight updates likely has to do with the fact that it uses a single scalar learning rate for each matrix, which prevents any adaptive learning rates from becoming too small, whereas Adam seems to end up producing many very small, per-parameter learning rates.}\label{fig:wtAnlyzConv}
\end{figure}

\begin{figure}[b]
\includegraphics[width=\textwidth]{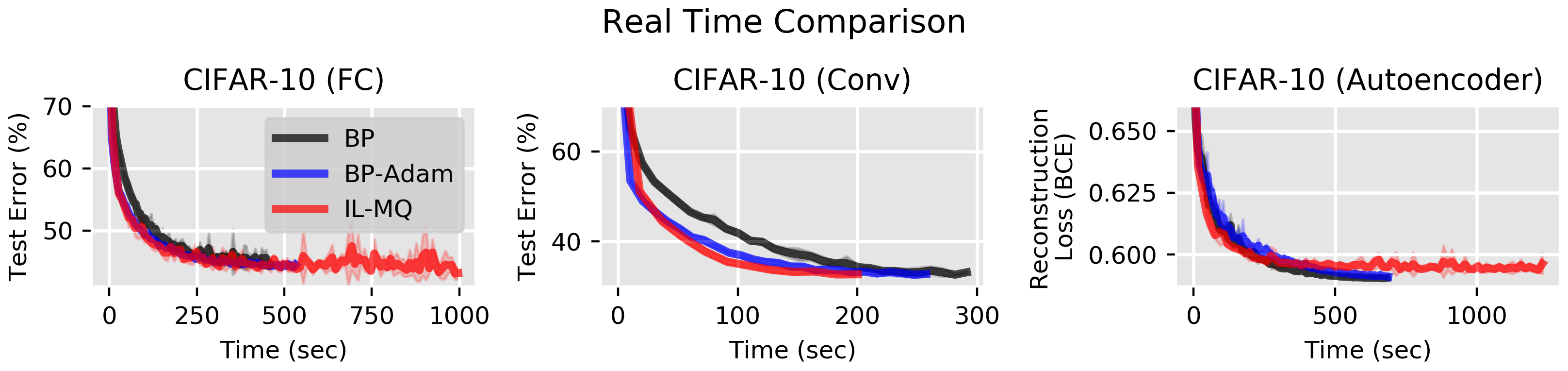}
\caption{MLP, convolutional network, and AE training runs shown for Cifar-10 with x-axis showing real time in seconds. Training runs shown to convergence or max number of epochs.}
\centering
\label{fig:RealTime}
\end{figure}

\begin{table}[h]
\parbox{.33\linewidth}{
\centering
\begin{tabular}{ccc}
\toprule
    \multicolumn{2}{c} {Memory Requirements}\\
    \toprule
    Algorithm & Memory \\
    \toprule
    BP & $\mathcal{O}(M)$ \\
    BP-Adam & $\mathcal{O}(3M)$ \\
    IL-MQ & $\mathcal{O}(M)$ \\
    IL-Adam & $\mathcal{O}(3M)$ \\
    \bottomrule
\end{tabular}
\caption{Memory computed as model parameters (M) + optimizer/algorithm parameters, in the limit where $M \rightarrow \infty$ with finite network layers.}\label{tab:MemComp}
}
\hfill
\parbox{.63\linewidth}{
\centering
\begin{tabular}{c c c c}
    \toprule
    \multicolumn{4}{c} {Computation Requirements}\\
    \toprule
    & BP & SeqIL & IL(T=15) \\
    \toprule
    MLP & $3L-1$ & $7L-1$ & $31L-1$\\
    AutoEnc. & $3L-1$ & $13L-1$ & $31L-1$\\
    \bottomrule
\end{tabular}
\caption{The number of matrix multiplies performed each training iteration. L is the number weight matrices. SeqIL computation is computed given the value T used in our simulations above. We compare this to standard IL that uses a more common value of T (15) (see app. \ref{app:methods} for details).}\label{tab:BPILComp}
}
\end{table}

\begin{table}[h]
\centering
\begin{tabular}{c c}
    \toprule
    \multicolumn{2}{c} {Computation Requirements}\\
    \toprule
    IL/SeqIL & iPC \\
    \toprule
    $T (2L-1) + L$ & $T(2L-1) + LT$\\
    \bottomrule
\end{tabular}
\caption{The number of matrix multiplications per training iteration, where L is the number weight matrices and T the number of inference iterations. Each inference iteration, 2L-1 matrix multiplies are needed to update each $\hat{h}_l$. With T inference iterations, this means T(2L-1) matrix multiplies are needed to perform inference. Since IL and SeqIL only update weights once per training iteration, another L matrix multiples are needed to compute weight updates for each weight matrix, resulting in a total of T(2L-1) + L matrix multiplies. Incremental PC (iPC) \cite{salvatori2022incremental} updates all weights each inference iteration requiring TL matrix multiplies, for a total of T(2L-1) + TL. Note, however, that \cite{salvatori2022incremental} 
 et al., point out that iPC typically requires fewer matrix multiplies to perform on non-zero update on each matrix, since weights are updated each iteration. This may be a useful property for scenarios where, e.g. the input is presented in close to real time and weight must be updated rapidly.}\label{tab:iPCIL}
\end{table}

\end{document}